%% file: main.tex
\documentclass[sigconf]{acmart}

\input{header}

\setcopyright{none}
\copyrightyear{2026}
\acmYear{2026}
\acmDOI{}

\acmISBN{}

\begin{document}

\title{Adaptive Data Admission and Retention for Streaming\\ Federated Learning}

\author{Zhuoyi Zhao}
\affiliation{%
  \institution{University of Toronto}
  \city{Toronto, ON}
  \country{Canada}
}
\email{zhuoyijoeyzhao@gmail.com}

\author{Ben Liang}
\affiliation{%
  \institution{University of Toronto}
  \city{Toronto, ON}
  \country{Canada}
}
\email{liang@ece.utoronto.ca}

\renewcommand{\shortauthors}{Zhao et al.}

\begin{abstract}
\input{Body/0_Abstract}
\end{abstract}

\ccsdesc[500]{Computing methodologies~Machine learning}
\ccsdesc[300]{Computing methodologies~Distributed computing methodologies}
\ccsdesc[100]{Networks~Network dynamics}

\keywords{streaming federated learning, online sampling, stochastic optimization}

\maketitle

\input{Body/1_Intro}

\input{Body/2_Model}
\input{Body/3_Convergence_Analysis}

\input{Body/5_ADPP_policy}

\input{Body/5_1_Performance_guarantee}
\input{Body/7_Experiments}
\input{Body/6_Conclusion}

\input{Body/8_Impact}

\nocite{langley00}

\bibliographystyle{ACM-Reference-Format}
\bibliography{references}
\newpage
\setcounter{page}{1}
\twocolumn[
\begin{@twocolumnfalse}
{
\begin{center}
\fontsize{18pt}{25pt}\selectfont Supplementary Material for the paper ``Adaptive Data Admission and Retention for Streaming Federated Learning''\\
\vspace{0.5cm}
\large Zhuoyi Zhao and Ben Liang
\vspace{0.5cm}
\end{center}
}
\end{@twocolumnfalse}]
\appendix
\input{Appendix/Appendix_A}
\input{Appendix/Appendix_C}

\end{document}

%% file: header.tex
\usepackage{microtype}
\usepackage{graphicx}
\usepackage{booktabs}
\usepackage{subcaption}
\usepackage{hyperref}

\usepackage{amsmath}

\usepackage{amssymb}
\usepackage{mathtools}
\usepackage{amsthm}
\usepackage[ruled,vlined,linesnumbered]{algorithm2e}
\DontPrintSemicolon
\SetKwInput{KwInput}{Inputs}
\SetKwInput{KwOutput}{Output}
\SetKwInput{KwInit}{Initialization}
\SetKwBlock{ServerBlock}{The server does:}{}
\SetKwBlock{ClientBlock}{Each client $m$ in parallel does:}{}
\usepackage[capitalize,noabbrev]{cleveref}

\usepackage[textsize=tiny]{todonotes}

\theoremstyle{plain}
\newtheorem{theorem}{Theorem}

\newtheorem{lemma}{Lemma}
\newtheorem{corollary}{Corollary}

\theoremstyle{definition}
\newtheorem{definition}{Definition}
\newtheorem{assumption}{Assumption}

\theoremstyle{remark}
\newtheorem{remark}{Remark}

%% file: Body/0_Abstract.tex
We study streaming federated learning with limited client memory, where newly generated training data incur time-varying sampling costs and must be selectively admitted and retained over time. We consider a joint server-side admission and client-side memory-management framework with the objective of minimizing the cumulative excess population risk under a sampling-cost budget and buffer constraints. We first derive a learning-error bound that explicitly captures the effects of instantaneous training sample size, distinct-sample growth, and reuse imbalance through a characterization of the effective sample size. Through a surrogate penalty obtained from this bound, we develop an Active-Constraint Drift-Plus-Penalty (ACDPP) policy that combines a structured client-side $K$-step retention rule with a server-side online admission rule and a time-varying rectangular admission region. We further present a sequence of comparison arguments, via an auxiliary constant-admission policy, that connects the ACDPP learning bound to a costless oracle benchmark.
This yields explicit guarantees in terms of sublinear regret and sampling-cost violation, while the buffer-occupancy violation is controlled through offline selection of the retention horizon. Experiments on multiple datasets demonstrate that the proposed policy remains close to the oracle benchmark while satisfying the sampling-cost and buffer constraints.

%% file: Body/1_Intro.tex
\section{Introduction}
\label{sec:intro}

Federated learning (FL) enables multiple devices to collaboratively train a shared model without directly transferring their raw data to a central server~\cite{mcmahan2017communication,kairouz2021advances}. 
By keeping data local, FL has emerged as a promising paradigm for privacy-aware and large-scale model training in mobile and edge systems~\cite{bonawitz2019towards,wang2019adaptive}. 
The performance of FL, however, depends fundamentally on the quality, quantity, and diversity of the data available across participating clients. 
In practice, client data are often highly heterogeneous due to differences in user behavior, sensing environments, and local data-generation processes and generalization~\cite{karimireddy2020scaffold,tan2022towards,wang2020tackling}. 
Moreover, in many real-world applications, local datasets evolve over time as new samples are continuously generated, collected, or labeled~\cite{gong2023ode,shi2022sofederated,sun2025learn,zhang2024federating,huynh2025streaming}. 
As a result, beyond model aggregation itself, an important question in FL is how to manage the training data pipeline under statistical heterogeneity and temporal data evolution.

In particular, client data may be generated and collected in an online fashion rather than provided as a fixed offline dataset. 
Examples include edge sensing, mobile crowdsourcing, and human-in-the-loop labeling systems, where new samples arrive sequentially and may incur non-negligible acquisition, annotation, or transmission costs before they can be used for training~\cite{zhou2020cefl,hu2024energy}. 
Such costs are often time-varying due to changes in sensing conditions, network status, device availability, or labeling effort. 
At the same time, client devices are typically memory-limited and cannot retain all previously observed samples indefinitely. 
As a result, the learner must continuously decide whether to admit newly arrived samples and which historical samples to retain for future reuse. 
This creates a fundamental trade-off: admitting more fresh samples may improve data diversity and generalization, while reusing retained samples can reduce immediate acquisition cost and increase the effective training set size. 
Therefore, beyond classical federated optimization, streaming FL calls for a principled framework that manages sample admission and buffer occupancy under dynamic cost and memory constraints.

While prior works have studied resource management ~\cite{zhou2020cefl,hu2024energy}, and sample selection~\cite{gong2023ode,shi2022sofederated,sun2025learn,marfoq2023federated} in streaming FL, they do not provide a unified framework that directly optimizes learning performance under joint admission and memory constraints. 
In particular, existing resource-centric approaches mainly focus on communication, computation, or energy efficiency~\cite{zhou2020cefl,hu2024energy}, whereas learning-centric sample selection methods typically do not explicitly account for long-term sampling cost, buffer occupancy, or the coupled effect of sample admission and retention~\cite{gong2023ode,shi2022sofederated,sun2025learn}. 

Our work bridges this gap by proposing a learning-centric modeling framework for streaming FL with limited client memory, which jointly captures server-side sample admission and client-side buffer management. Specifically, we formulate the problem as a constrained stochastic optimization problem whose objective is to minimize the cumulative excess population risk under long-term sampling-cost and buffer-occupancy constraints. This formulation enables a principled treatment of the trade-off among sample freshness, sample reuse, memory limitation, and admission cost, and lays the foundation for both the benchmark design and the online policy developed in this paper.

The main contributions of this paper are summarized as follows.
\vspace{-1em}
\begin{itemize}
    \item We establish a learning-error bound for streaming FL
    with memory-based training, which explicitly characterizes the roles of the
    instantaneous training sample size, the cumulative number of distinct
    samples, and the reuse imbalance caused by repeated use of stored samples.
    In particular, we introduce an effective sample size representation that
    makes these effects transparent and amenable to control design.

    \item We propose an ACDPP policy that combines a structured
    client-side K-step retention rule with a server-side online admission rule
    based on a modified form of Lyapunov drift-plus-penalty optimization. In particular, to strengthen the
    control design, we impose a time-varying rectangular constraint on the
    feasible admission region, which gradually drives the admission process
    toward a target operating point.


    \item We prove explicit guarantees for the proposed policy, including a sublinear
    regret bound relative to a costless oracle benchmark and a sublinear sampling-cost
    violation bound derived via Lyapunov analysis. Our analysis proceeds through
    a sequence of finite-horizon comparisons that connect the ACDPP-induced
    learning bound to the oracle through its surrogate penalty and an auxiliary constant-admission policy. In addition, we provide an offline method for choosing the time-invariant retention horizon K to ensure uniformly bounded cumulative buffer-occupancy violation.

    \item We conduct experiments on multiple canonical datasets to validate the proposed
    policy and demonstrate its effectiveness under streaming sample admission,
    buffer constraints, and time-varying sampling costs.
\end{itemize}

\vspace{-1em}
\section{Related Work}
\label{sec:related work}
\subsection{Centralized Learning with Streaming Data}

In centralized streaming learning, samples arrive sequentially rather than as a fixed dataset. As a result, the learner typically trains on only a partially observed dataset at any point in time~\cite{gomes2019machine,gama2013evaluating}. This is often due to limited storage or the use of memory-efficient stream summaries in high-volume data streams~\cite{cormode2005improved,
wang2024comprehensive,de2021continual}. Despite these connections, centralized streaming learning does not capture the network-edge setting considered here. It overlooks the distributed nature of data generation and storage across edge devices, along with the resulting heterogeneity in local data distributions, memory capacities, and resource constraints. It also generally does not account for the time-varying operating conditions that are common in edge environments.
\vspace{-1em}
\subsection{Streaming Federated Learning}
\label{subsec: related works}
Several works have studied online resource management for streaming FL using Lyapunov optimization. CEFL \cite{zhou2020cefl} jointly optimizes admission control, load balancing, and accuracy tuning to minimize operational costs. Hu et al. \cite{hu2024energy} further incorporate device scheduling and bandwidth allocation under long-term energy constraints. While these approaches enable online decision-making with constraint guarantees, they primarily focus on system-level costs rather than learning performance, and assume unlimited on-device storage.

Another line of research develops data valuation metrics for sample selection under limited on-device storage. Gong et al. \cite{gong2023ode} propose gradient-norm based selection to accelerate convergence, Shi et al. \cite{shi2022sofederated} employ contrastive learning for importance scoring in unlabeled settings, and Sun et al. \cite{sun2025learn} learn query policies via multi-agent reinforcement learning. Notably, while these score-based methods implicitly assume that high-gradient-norm or high-loss samples are more valuable for learning, they do not explicitly optimize for the underlying population risk. In contrast, our framework directly minimizes the cumulative excess risk with respect to the target distribution, providing theoretical regret guarantees. Furthermore, these scoring mechanisms focus on \textit{which} sample to select while our admission control determines \textit{how many} samples to retain under cost constraints.

Marfoq et al. \cite{marfoq2023federated} provide a foundational analysis for streaming FL, characterizing the bias-optimization tradeoff when mixing historical and fresh clients. They establish that the optimal strategy depends on the ratio between gradient variability and effective sample size, with the extreme cases being uniform weighting when gradient variability dominates and historical-only weighting when sample size effects dominate. However, their analysis does not address online admission decisions, storage management, or explicit optimization of the mixing strategy under resource constraints.


Recent studies also address non-stationary data streams where distributions evolve over time. Zhang et al. \cite{zhang2024federating} propose Fed-HIST to leverage historical model knowledge. Huynh et al. \cite{huynh2025streaming} analyze convergence under Markovian data streams. These works focus on adapting to distribution drift, whereas our work targets stationary environments where the optimal model remains fixed and the challenge lies in optimal sample management under resource constraints.

%% file: Body/2_Model.tex
\section{Preliminaries}
\label{sec:system-model}

\subsection{Network Model}
\label{sec:network-model}
We consider a streaming FL system consisting of a parameter
server (PS) and $M$ clients, operating over discrete communication rounds indexed
by $t\in\{1,2,\ldots,T\}$. Each client is indexed by
$m\in\{1,2,\ldots,M\}$. Training samples are generated in a streaming manner:
at the beginning of each round $t$, the PS determines the admission decisions,
and the admitted samples are stored locally at the clients for future training.\footnote{
An alternative is for the clients to locally determine admission. In that case, due to the positive sampling costs, it is necessary to consider fairness and incentives among the clients. Such game theoretic analysis is outside the scope of this work. Instead, here we focus on overall system cost and performance, so it suffices to consider server-side admission control.
}
We assume that client $m$ is associated with a local data distribution
$\mathcal{P}_m$, and that each sample admitted at client $m$ is drawn
independently and identically distributed (i.i.d.) from $\mathcal{P}_m$.

Due to privacy and security constraints, sample admission and local memory
management are decoupled across the server and the clients. In particular, the
PS never accesses raw samples; instead, it only determines the admission amount
at each client in each round. Once admitted, a sample is stored locally in the
buffer of the corresponding client. Buffer management is then carried out
independently by each client through local dropping actions based solely on
client-side information. The admitted and retained samples constitute the local
training memory used for subsequent local model updates.

\vspace{0.5ex}
\noindent\textbf{Server-side admission actions and sampling cost.}
At each round $t$, the PS decides how many new samples to admit at each client.
For each client $m$, samples are indexed by their generation order, and
$(m,i)$ denotes the $i$-th sample generated at client $m$. To represent the
sample-level admission status for subsequent buffer and age evolution, we define
the binary \emph{admission} indicator
\begin{equation}
a_{m,i}(t)\in\{0,1\},
\end{equation}
where $a_{m,i}(t)=1$ means that sample $(m,i)$ is admitted to the local buffer
of client $m$ in round $t$. Accordingly, we define the admission amount at
client $m$ in round $t$ as
\begin{equation}
\lambda_m(t)\triangleq \sum_i a_{m,i}(t),\quad\forall m,
\label{eq:lambda-def}
\end{equation}
which is the number of samples admitted at client $m$ in round $t$.

Admitting one new sample in round $t$ incurs a time-varying per-sample
\emph{sampling cost} $c(t)\ge 0$, revealed to the PS at the beginning of round
$t$. Let
$\bar c_T \triangleq \frac{1}{T}\sum_{t=1}^T \mathbb{E}[c(t)]$
denote the average per-sample sampling cost over the horizon $T$. We impose the average sampling-cost constraint
\begin{equation}
\frac{1}{T}\sum_{t=1}^T
\mathbb{E}\!\left[c(t)\sum_{m=1}^M \lambda_m(t)\right]
\le C,
\label{eq:sampling-cost-constraint}
\end{equation}
where $C>0$ is the prescribed per-round sampling-cost budget.

\vspace{0.5ex}
\noindent\textbf{Client memory state and buffer dropping.}
Each client maintains a local buffer that stores admitted samples for future
training. Let $q_{m,i}^t\in\{0,1\}$ indicate whether sample $(m,i)$ is present in
the local buffer of client $m$ at the beginning of the local training phase of
round $t$, i.e., after the admission and dropping actions in round $t$ have been
executed. The resulting buffer occupancy is defined as
\begin{equation}
n_m(t)\triangleq \sum_i q_{m,i}^t,
\label{eq:buffer-occupancy}
\end{equation}
which is the number of samples stored in the local buffer of client $m$ in round
$t$. Accordingly, let $\mathcal{Q}_m(t)$ denote the set of samples stored in the
local buffer of client $m$ in round $t$, so that $n_m(t)=|\mathcal{Q}_m(t)|$. To capture limited memory, we impose an average
buffer-occupancy constraint on each client:
\begin{equation}
\frac{1}{T}\sum_{t=1}^T \mathbb{E}[n_m(t)] \le B_m,
\qquad m=1,\ldots,M,
\label{eq:avg-buffer-constraint}
\end{equation}
where $B_m>0$ is the prescribed average memory budget of client $m$.

Client $m$ independently manages its local buffer through a binary
\emph{drop} action $d_{m,i}(t)\in\{0,1\}$, where $d_{m,i}(t)=1$ means that
sample $(m,i)$ is removed from the local buffer of client $m$ in round $t$.
Dropping actions are determined locally based solely on client-side
information.  

\vspace{0.5ex}
\noindent\textbf{Sample age and its evolution.}
For each sample $(m,i)$, let $A_{m,i}^t\in\mathbb{N}$ denote its \emph{age} at
the beginning of the local training phase of round $t$. The age serves as a
residence-time counter and is updated only when the sample is present in the
local buffer at the beginning of the next local training phase. Specifically,
\begin{equation}
A_{m,i}^{t+1}=
\begin{cases}
A_{m,i}^{t}+1, & \text{if } q_{m,i}^{t}=1 \text{ and } q_{m,i}^{t+1}=1,\\
1, & \text{if } q_{m,i}^{t}=0 \text{ and } q_{m,i}^{t+1}=1,\\
A_{m,i}^{t}, & \text{if } q_{m,i}^{t+1}=0,
\end{cases}
\label{eq:age-evolution}
\end{equation}
with initialization $A_{m,i}^{0}=0$ for all $m,i$. Thus, the age of a sample is
incremented while it remains in the buffer, initialized to $1$ when it newly
enters the buffer, and frozen once it is no longer stored.





\subsection{Distributed Learning Model}
\label{sec:learning-model}

Let $\theta^{(t)} \in \Theta$ denote the global model at the beginning of the
local training phase of round $t$. We assume that all clients participate in every communication round.
Moreover, due to the limited buffer size at each client, we consider a
\emph{full-batch local update} regime, in which each client directly uses all
samples currently stored in its local buffer to perform local training.

Let $\ell(\theta;z)$ denote the sample-wise loss of model parameter $\theta$ on
data sample $z$. The empirical risk over the local buffer of client $m$ in round
$t$ is defined as
\begin{equation}
\widehat{F}_{m,t}(\theta)
\triangleq
\frac{1}{n_m(t)}
\sum_{z \in \mathcal{Q}_{m}(t)} \ell(\theta;z).
\label{eq:local-empirical-risk}
\end{equation}
Accordingly, the corresponding local gradient is
\begin{equation}
\nabla \widehat{F}_{m,t}(\theta)
=
\frac{1}{n_m(t)}
\sum_{z \in \mathcal{Q}_{m}(t)}
\nabla_\theta \ell(\theta; z).
\label{eq:local-gradient}
\end{equation}

Starting from $\theta_m^{(t,0)} = \theta^{(t)}$, each client performs $E \ge 1$
steps of gradient descent with stepsize $\eta > 0$ on the empirical risk
$\widehat{F}_{m,t}$:
\begin{equation}
\theta_m^{(t,e+1)}
=
\theta_m^{(t,e)}
-
\eta\, \nabla \widehat{F}_{m,t}\!\left(\theta_m^{(t,e)}\right),
\qquad e = 0,\ldots,E-1.
\label{eq:local-gd}
\end{equation}

After the local updates, client $m$ uploads the model difference
\begin{equation}
\Delta_m^{(t)} \triangleq \theta_m^{(t,E)} - \theta^{(t)}.
\label{eq:local-update}
\end{equation}
Let $\alpha_m\ge 0$ denote the prescribed importance weight of client $m$, with
$\sum_{m=1}^M \alpha_m = 1$.
The server then aggregates the client updates according to a FedAvg-style rule:
\begin{equation}
\theta^{(t+1)}
=
\theta^{(t)}
+
\sum_{m=1}^M w_{m,t}\, \Delta_m^{(t)},
\qquad
w_{m,t} \ge 0,\quad \sum_{m=1}^M w_{m,t} = 1,
\label{eq:global-aggregation}
\end{equation}
where a natural choice of $w_m,t$, for example, is memory-proportional weighting, i.e.
\begin{equation}
w_{m,t}
=
\frac{\alpha_m n_m(t)}{\sum_{j=1}^M \alpha_j n_j(t)}.
\label{eq:memory-proportional-weight}
\end{equation}

\subsection{Problem Formulation}
\label{sec:problem-formulation}

\vspace{0.5ex}
\noindent\textbf{Learning objective.}
We define
the \emph{target population} distribution as the weighted mixture
\begin{equation}
\mathcal{P} \triangleq \sum_{m=1}^M \alpha_m \mathcal{P}_m.
\label{eq:population-def}
\end{equation}
Our goal is to learn a model that minimizes the population risk
\begin{equation}
F_{\mathcal{P}}(\theta)
\triangleq
\mathbb{E}_{z \sim \mathcal{P}}[\ell(\theta;z)],
\qquad \theta \in \Theta,
\label{eq:population-risk}
\end{equation}
and let $\theta^\star \in \arg\min_{\theta \in \Theta} F_{\mathcal{P}}(\theta)$.

\vspace{0.5ex}
\noindent\textbf{Server and client control variables.}
The system is controlled through a decomposed architecture. On the server side,
the PS determines the admission amounts $\{\lambda_m(t)\}_{m=1}^M$ at each round
$t$, which can be equivalently represented by the sample-level admission
indicators $\{a_{m,i}(t)\}_{m,i}$ via~\eqref{eq:lambda-def}. On the client side,
each client $m$ independently determines the sample-level dropping actions
$\{d_{m,i}(t)\}_i$ to manage its local buffer. Together, the admission and
dropping actions govern the evolution of the client buffers and must jointly
satisfy the sampling-cost constraint~\eqref{eq:sampling-cost-constraint} and
the average buffer-occupancy constraint~\eqref{eq:avg-buffer-constraint}.

Let $\pi^{\mathrm S}$ denote the server-side policy that maps
server-observable information (e.g., the history of communicated updates and the
current cost $c(t)$) to admission actions. For each client $m$, let
$\pi_m^{\mathrm C}$ denote the client-side policy that maps local buffer states
(including sample ages) to dropping actions. We denote the resulting joint
policy by
\begin{equation}
\pi \triangleq \bigl(\pi^{\mathrm S},\{\pi_m^{\mathrm C}\}_{m=1}^M\bigr).
\label{eq:joint-policy}
\end{equation}

\vspace{0.5ex}
\noindent\textbf{Optimization objective and constraint violation.}
Under the learning dynamics in Section~\ref{sec:learning-model}, the global
model sequence $\{\theta^{(t)}\}_{t=1}^T$ is induced by the joint policy $\pi$
through its effect on sample admissions, buffer evolution, and the resulting
local updates. Although performance is evaluated with respect to the target
population distribution $\mathcal{P}$, the model at each round is trained only
on the samples currently stored in the client buffers. These stored samples are
time-varying and are jointly determined by the admission and dropping actions
over time. Hence, policy design in our setting amounts to controlling the
evolving buffer contents so that training on the stored samples leads to low
expected loss under the target population distribution $\mathcal{P}$.

To evaluate the learning performance induced by a given joint policy $\pi$, we
define the \emph{cumulative excess population risk} up to horizon $T$ as
\begin{equation}
\mathcal{E}_T(\pi)
\triangleq
\sum_{t=1}^{T}
\mathbb{E}\!\left[
F_{\mathcal{P}}(\theta^{(t)}) - F_{\mathcal{P}}(\theta^\star)
\right].
\label{eq:cum-excess-risk}
\end{equation}
This quantity measures the cumulative suboptimality of the model sequence
induced by policy $\pi$ relative to the population-risk minimizer
$\theta^\star$.

Our goal is to develop joint policies $\pi$ that minimize the cumulative excess
population risk $\mathcal{E}_T(\pi)$ subject to the sampling-cost
constraint~\eqref{eq:sampling-cost-constraint} and the average
buffer-occupancy constraint~\eqref{eq:avg-buffer-constraint}.

%% file: Body/3_Convergence_Analysis.tex
\section{Streaming FL Convergence Analysis}
\label{sec:convergence}

In this section, we analyze the learning performance of the streaming FL under server-side sample admission and client-side memory
management. The key challenge is that the model is trained on time-varying
buffer contents induced by the joint policy, rather than on direct samples from
the target population distribution. Our focus is therefore on understanding how
the evolving stored samples and their repeated reuse affect the convergence
behavior of the learning process. In particular, we derive a bound on the
cumulative excess population risk in~\eqref{eq:cum-excess-risk}.




\subsection{Assumptions}
\label{subsec:assumptions}

We impose the
following standard regularity assumptions.

\begin{assumption}[Bounded domain]
\label{ass:bounded-domain}
The parameter set $\Theta$ has diameter $D$, i.e.,
\[
\|\theta-\theta'\|_2 \le D,
\qquad \forall \theta,\theta'\in\Theta.
\]
\end{assumption}

\begin{assumption}[Smoothness and bounded gradients]
\label{ass:smooth-grad}
The population risk $F_{\mathcal{P}}$ is $L$-smooth on $\Theta$, i.e.,
\[
\|\nabla F_{\mathcal{P}}(\theta)-\nabla F_{\mathcal{P}}(\theta')\|_2
\le L\|\theta-\theta'\|_2,
\qquad \forall \theta,\theta'\in\Theta.
\]
Moreover, the loss $\ell(\theta;z)$ is differentiable in $\theta$, and its
per-sample gradients are uniformly bounded:
\[
\|\nabla_\theta \ell(\theta;z)\|_2 \le G,
\qquad \forall \theta\in\Theta,\; z.
\]
Hence, any empirical gradient formed as an average over stored samples is also
bounded in norm by $G$.
\end{assumption}

\begin{assumption}[Bounded loss and finite pseudo-dimension]
\label{ass:bounded-loss-pdim}
The loss is bounded as $\ell(\theta;z)\in[0,B]$ for all $\theta\in\Theta$ and
$z$. In addition, the loss-composed hypothesis class $\ell\circ\mathcal{H}$ has
finite pseudo-dimension, denoted by $\mathrm{Pdim}(\ell\circ\mathcal{H})$.
\end{assumption}

For sharper control of the stochastic-gradient term, we further impose the following assumption.
\begin{assumption}[Bounded gradient variance]
\label{ass:grad-noise}
The per-sample gradient has uniformly bounded population variance:
\begin{equation}
\label{eq:pop-grad-var}
\mathbb{E}_{z\sim\mathcal{P}}
\left[
\left\|
\nabla_\theta \ell(\theta;z)-\nabla F_{\mathcal{P}}(\theta)
\right\|_2^2
\right]
\le \sigma^2,
\qquad \forall \theta\in\Theta.
\end{equation}
\end{assumption}

\subsection{Learning Error Bound}
\label{subsec:regret-bound}

We next derive an upper bound on the cumulative excess population risk
$\mathcal{E}_T(\pi)$ under a given joint policy $\pi$. The bound depends on the
policy-induced evolution of the training samples, in particular on the
instantaneous training sample size, the cumulative number of distinct samples,
and the loss of diversity caused by repeated reuse of stored samples.

Under the full-batch local training regime, let $n^\pi(t)$ denote the total
number of stored samples used for training at round $t$ under policy $\pi$.
Further, let $N^\pi(t)$ denote the number of \emph{distinct} samples that have
been used in local training up to round $t$. Since retained samples may be
reused multiple times across rounds, the quantity $N^\pi(t)$ alone does not
fully reflect the amount of statistical information available for learning. To
capture the loss of diversity caused by uneven sample reuse, we introduce an
\emph{effective sample size}.

Specifically, for each distinct sample $i\in\{1,\ldots,N^\pi(t)\}$, let
$A_i^\pi(t)$ denote the cumulative number of times that sample $i$ has been used
in local training up to round $t$. We define
\begin{equation}
\widetilde{\mathcal V}^\pi(t)
\triangleq
\frac{\bar{A}^\pi(t)^2}{\bar{A}^\pi(t)^2+\mathrm{Var}(A^\pi(t))},
\label{eq:Vtilde-def-body}
\end{equation}
and
\begin{equation}
\label{eq:Neff-def}
N_{\mathrm{eff}}^\pi(t)
\triangleq
N^\pi(t)\,\widetilde{\mathcal V}^\pi(t),
\end{equation}
where $\bar{A}^\pi(t)$ and $\mathrm{Var}(A^\pi(t))$ are the empirical mean and
variance of $\{A_i^\pi(t)\}_{i=1}^{N^\pi(t)}$, respectively. The factor
$\widetilde{\mathcal V}^\pi(t)\in(0,1]$ measures the uniformity of sample reuse:
it is close to $1$ when the reuse counts are nearly uniform, and becomes smaller
when the reuse pattern is highly uneven. Accordingly,
$N_{\mathrm{eff}}^\pi(t)=N^\pi(t)\widetilde{\mathcal V}^\pi(t)$ quantifies the
effective amount of statistical diversity retained in the stored samples.

We now state the main excess-risk bound. The proof combines
(i) a smoothness-based optimization recursion, (ii) variance control for
stochastic gradients under sample reuse, and (iii) a uniform generalization
bound controlled by the pseudo-dimension; detailed arguments are deferred to the
appendix.

The bound is closely related in spirit to the general streaming-learning
analysis of Theorem~4.4 in~\cite{marfoq2023federated}, but is specialized here
to the present full-batch memory-based federated setting. Rather than giving an
abstract order-level characterization as in~\cite{marfoq2023federated}, we express the bound explicitly in terms
of the policy-induced quantities $n^\pi(t)$, $N^\pi(t)$, and
$\widetilde{\mathcal V}^\pi(t)$ through
$N_{\mathrm{eff}}^\pi(t)=N^\pi(t)\widetilde{\mathcal V}^\pi(t)$. This explicit
decomposition makes the separate effects of training sample size, distinct-sample
growth, and reuse imbalance transparent, and will be instrumental for the
subsequent online control design.
\vspace{-0.5em}
\begin{theorem}
\label{thm:regret-neff}
Under Assumptions~\ref{ass:bounded-domain}--\ref{ass:grad-noise}, let
$\{\theta_\pi^{(t)}\}_{t=1}^T$ be the model sequence generated by the learning
dynamics under a given joint policy $\pi$. Then the cumulative excess
population risk satisfies
\begin{equation}
\label{eq:cum-regret-bound}
\mathcal{E}_T(\pi)
\le
\bar{\mathcal{E}}_T(\pi),
\end{equation}
where\vspace{-0.5ex}
\begin{equation}
\label{eq:cum-regret-bound_detail}
\begin{aligned}
\bar{\mathcal{E}}_T(\pi)
&\!=
\!\sum_{t=1}^T
\!\Bigg[
D\sigma
\!\!\sqrt{
\frac{1}{n^\pi(t)}\!-\!\frac{1}{N^\pi\!(t)\widetilde{\mathcal V}^\pi(t)}
}
\!+\!
\eta\sigma^2\!
\left(\!\!
\frac{1}{n^\pi(t)}\!-\!\frac{1}{N^\pi\!(t)\widetilde{\mathcal V}^\pi(t)}
\!\right)
\\
&\quad+
10B
\sqrt{
\frac{\mathrm{Pdim}(\ell\circ\mathcal H)}
{N^\pi\!(t)\widetilde{\mathcal V}^\pi(t)}
}
\sqrt{
1+\log\!\left(
\frac{N^\pi\!(t)}{\mathrm{Pdim}(\ell\circ\mathcal H)}
\right)
}
+
C_2
\Bigg].
\end{aligned}
\end{equation}
Here, $C_2$ collects initialization and lower-order constant terms, including
$\frac{D^2}{2\eta}$ and other algorithm-dependent constants.
\end{theorem}
\vspace{-1em}
\begin{proof}
See Appendix~\ref{app: Proof of regret bound} in~\cite{adaptive_data_admission_retention_sfl_supp}.
\end{proof}

%% file: Body/5_ADPP_policy.tex
\vspace{-1.5em}
\section{Active-Constraint DPP Policy}
\label{sec:ACDPP}

Utilizing the excess-risk bound in~\eqref{eq:cum-regret-bound}, we develop a joint admission-retention design for the
streaming FL system. Our goal is to construct a structured
online control policy that satisfies the long-term sampling-cost constraint and
minimizes a tractable surrogate of the cumulative excess population risk.

While Theorem~\ref{thm:regret-neff} provides an explicit policy-dependent upper
bound, it remains difficult to optimize directly under a general
buffer-management policy, because the reuse-uniformity factor
$\widetilde{\mathcal V}^\pi(t)$ depends on the entire sample-retention pattern.
To address this difficulty, we first impose a structured client-side retention
rule, namely the K-step retention policy, which regularizes sample reuse and
yields a tractable surrogate for online control.
Based on this client-side structure, we then construct a DPP-based server-side
admission policy using a cost-debt virtual queue together with a
learning-oriented penalty term. To further strengthen the control design, we
impose a \emph{time-varying rectangular constraint} on the admission action
space. This construction extends the classical DPP framework
in~\cite{neely2022stochastic} by allowing the feasible admission interval to
shrink gradually over time toward the target operating point. 

Algorithm~\ref{alg:acdpp} summarizes the overall Active-Constraint DPP policy,
including the client-side K-step retention rule, the server-side online
admission update, and the resulting distributed learning procedure. We next
formalize each component of this policy.






\begin{algorithm}[t]
\caption{Active-Constraint DPP Algorithm}
\label{alg:acdpp}

\KwInput{$\theta^{(1)}$, $Z(1)=0$, $\{q_{m,i}^0\}$, $\{A_{m,i}^0\}$, $N(0)$, $n(0)$, $C$, $\{B_m\}_{m=1}^M$, $K$, $V$, $\bar{\Lambda}$, $\rho$}
\KwOutput{$\{\theta^{(t)}\}_{t=1}^{T+1}$}

\For{$t=1,2,\ldots,T$}{

\ServerBlock{
Observe $\mathbb{O}(t)$ and $c(t)$\;

Solve~\eqref{eq:dpp-policy} to obtain $\boldsymbol{\lambda}^\star(t)$\;

Update $Z(t+1)$ and $\{A_{m,i}^{t+1}\}$ using~\eqref{eq:Z-update} and~\eqref{eq:age-evolution}\;

Update $N^{\pi}(t)$ and $n^{\pi}(t)$\;
}

\ClientBlock{
Admit $\lambda_m(t)$ new samples using $\boldsymbol{\lambda}^\star(t)$\;

Perform local training and compute $\Delta_m^{(t)}$ using \eqref{eq:local-gradient}, \eqref{eq:local-gd}, and \eqref{eq:local-update}\;

Update $\{q_{m,i}^{t+1}\}$ using~\eqref{eq:kstep-drop}\;
}

\ServerBlock{
Aggregate and update $\theta^{(t+1)}$ using~\eqref{eq:global-aggregation}\;
}
}
\end{algorithm}
\vspace{-0.5em}
\subsection{K-Step Retention Policy}
\label{subsec:kstep-structure}

We first introduce the K-step retention policy, which imposes a simple
deterministic client-side retention rule and makes the reuse pattern of stored
samples explicitly characterizable. The key idea is that, under K-step
retention, every admitted sample remains in the local buffer for a fixed number
of rounds, so that the reuse-uniformity factor
$\widetilde{\mathcal V}^\pi(t)$ can be expressed in closed form.

\begin{definition}[K-step retention policy]
\label{def:kstep}
Fix an integer retention horizon $K\ge 1$. A client-side policy is called a
\emph{K-step retention policy} if, for every client
$m\in\{1,\ldots,M\}$, every admitted sample $(m,i)$, and every round
$t\in\{1,\ldots,T\}$, the dropping action satisfies
\begin{equation}
d_{m,i}(t)
=
\mathbf{1}\{A_{m,i}^t = K\},
\qquad \forall m,i,t.
\label{eq:kstep-drop}
\end{equation}
That is, each sample is retained in the local buffer for exactly $K$ rounds
after admission and is removed when its age reaches $K$.
\end{definition}


We next characterize the reuse-uniformity factor under the K-step retention policy. Under this policy, every sample admitted no later than
round $t-K+1$ has already contributed exactly $K$ training uses by round $t$,
while only the samples admitted in the most recent $K-1$ rounds incur a reuse
deficit. Define
\begin{equation}
D_{\partial}^{(1)}(t)
\triangleq
\sum_{j=t-K+2}^{t}\bigl(K-(t-j+1)\bigr)\sum_{m=1}^{M}\lambda_m(j),
\label{eq:boundary-deficit-1}
\end{equation}
and
\begin{equation}
D_{\partial}^{(2)}(t)
\triangleq
\sum_{j=t-K+2}^{t}\bigl(K^2-(t-j+1)^2\bigr)\sum_{m=1}^{M}\lambda_m(j).
\label{eq:boundary-deficit-2}
\end{equation}
Then the reuse-uniformity factor under the K-step retention policy admits the
exact representation
\begin{equation}
\widetilde{\mathcal V}^{\pi}(t)
=
\frac{\bigl(KN^{\pi}(t)-D_{\partial}^{(1)}(t)\bigr)^2}
{N^{\pi}(t)\bigl(K^2N^{\pi}(t)-D_{\partial}^{(2)}(t)\bigr)}.
\label{eq:kstep-vtilde-exact}
\end{equation}
Hence, the finite-horizon deviation of $\widetilde{\mathcal V}^{\pi}(t)$ from
one is determined entirely by the boundary deficits induced by the samples
admitted in the most recent $K-1$ rounds.
\vspace{-0.5em}
\subsection{Online Admission Control with Time-Varying Rectangular Constraints}

We next formalize the server-side online admission control, including the
virtual queue, the Lyapunov function, the penalty term, and a new DPP
admission rule with time-varying rectangular constraints.

\vspace{0.5ex}\noindent\textbf{Cost-Debt Queue}.  
Let $Z(t)$ denote the cost-debt queue at the beginning of round $t$. It evolves as
\begin{equation}
Z(t+1)=\bigl[ Z(t)+C_t-C \bigr]^+,
\label{eq:Z-update}
\end{equation}
where $C>0$ is the prescribed per-round sampling-cost budget and
\begin{equation}
C_t \triangleq c(t)\Lambda(t),
\qquad
\Lambda(t)\triangleq \sum_{m=1}^M \lambda_m(t),
\end{equation}
is the aggregate sampling cost incurred in round $t$ under the DPP policy, with
initialization $Z(0)=0$.

The queue $Z(t)$ measures the accumulated violation of the budget constraint:
when $C_t>C$, the system incurs positive debt and $Z(t)$ grows; otherwise the
debt is reduced. Hence, a large value of $Z(t)$ indicates that the system has
been operating above the desired average cost level.

\vspace{0.5ex}\noindent\textbf{Lyapunov Function and Penalty Function}.  
We adopt the quadratic Lyapunov function
\begin{equation}
L(t)=\frac{1}{2}Z(t)^2,
\label{eq:lyapunov-dpp}
\end{equation}
which captures the evolution of the cost-debt dynamics.

Motivated by the learning-error bound in Theorem~\ref{thm:regret-neff} together
with the K-step retention structure developed in Section~\ref{subsec:kstep-structure},
we introduce the following per-round penalty, whose cumulative sum will serve
as a surrogate objective for online control:
\begin{equation}
\begin{aligned}
p(t)
&\triangleq 
D\sigma \sqrt{\frac{1}{n^{\pi}(t)}-\frac{1}{N^{\pi}(t)}}
+\tilde{\eta}\sigma^2\left(\frac{1}{n^{\pi}(t)}-\frac{1}{N^{\pi}(t)}\right)
\\
&\quad
+10B
\sqrt{\frac{\mathrm{Pdim}(\ell\circ\mathcal{H})}{N^{\pi}(t)}}
\sqrt{1+\log\!\left(\frac{N^{\pi}(t)}{\mathrm{Pdim}(\ell\circ\mathcal{H})}\right)}.
\end{aligned}
\label{eq:penalty-dpp}
\end{equation}
Accordingly, the cumulative quantity $\sum_{t=1}^{T} p(t)$ serves as a
surrogate objective in place of the learning upper bound
$\bar{\mathcal E}_T(\pi)$.

\vspace{0.5ex}\noindent\textbf{One-Slot Drift-Plus-Penalty}.  
Let
\[
\mathbb{O}(t) := \{ Z(t), q_{m,i}^t, A_{m,i}^t \}_{m,i}
\]
denote the system state at the beginning of the local training phase of round
$t$.\footnote{Under K-step retention, $q_{m,i}^t$ and $A_{m,i}^t$ are induced by
the historical admission process and are therefore inferable from server-side
actions. Their use in the drift analysis does not require access to raw client
data.} We define the one-slot Lyapunov drift as
\begin{equation}
\Delta(\mathbb{O}(t))
\triangleq
\mathbb{E}\!\left[ L(t+1)-L(t)\mid \mathbb{O}(t) \right].
\label{eq:lyapunov-drift}
\end{equation}

The DPP framework minimizes a weighted sum of the one-slot drift and the
per-round penalty. Accordingly, we consider the drift-plus-penalty expression
\begin{equation}
\Delta(\mathbb{O}(t))
+ V\mathbb{E}[p(t)\mid \mathbb{O}(t)],
\label{eq:dpp-objective}
\end{equation}
where $V>0$ is a control parameter that trades off learning performance and
budget satisfaction.

Substituting the queue evolution of $Z(t)$ into~\eqref{eq:lyapunov-drift} and
using standard quadratic-drift bounds yields
\begin{equation}
\Delta(\mathbb{O}(t))
+ V\mathbb{E}[p(t)\mid \mathbb{O}(t)]
\le
B
+
\mathbb{E}\!\left[ J_t(\boldsymbol{\lambda})\mid \mathbb{O}(t) \right],
\label{eq:dpp-upper-1}
\end{equation}
where $B<\infty$ is a uniform constant independent of the control action, and
\begin{equation}
J_t(\boldsymbol{\lambda})
\triangleq
Z(t)\bigl(c(t)\Lambda(t)-C\bigr)+Vp(t),
\label{eq:dpp-upper-2}
\end{equation}
with
\vspace{-0.5em}
\[
\boldsymbol{\lambda}(t)\triangleq (\lambda_1(t),\ldots,\lambda_M(t)).
\]

\vspace{0.5ex}\noindent\textbf{Acitve-Constraint DPP Admission Rule}.  
Motivated by the upper bound in~\eqref{eq:dpp-upper-1}, the ACDPP policy selects,
at each round $t$, an admission action vector $\boldsymbol{\lambda}^\star(t)$
that minimizes the right-hand side of~\eqref{eq:dpp-upper-2}, namely,
\vspace{-0.5em}
\begin{equation}
\boldsymbol{\lambda}^\star(t)
\in
\arg\min_{\boldsymbol{\lambda}}\; J_t(\boldsymbol{\lambda})
\quad \text{s.t.} \quad
\sum_{m=1}^M \lambda_m(t)\in[\Lambda_{\min}(t),\Lambda_{\max}(t)].
\label{eq:dpp-policy}
\end{equation}
Here, the bounds $\Lambda_{\min}$ and $\Lambda_{\max}$ impose a rectangular
constraint on the aggregate admission rate.

Unlike the standard DPP policy, we set time-varying constraints as follows:
\vspace{-0.5em}
\begin{equation}
\Lambda_{\min}(t)=\bar\Lambda(1-\rho^t),
\qquad
\Lambda_{\max}(t)=\frac{\bar\Lambda}{1-\rho^t},
\label{eq:lambda-min-max}
\end{equation}
for some constant $\rho\in(0,1)$. As $t\to\infty$, the admissible interval
shrinks to the singleton $\{\bar{\Lambda}\}$. The parameter $\rho$ controls the
contraction speed: smaller $\rho$ leads to faster concentration, while larger
$\rho$ preserves more flexibility for a longer period. We will discuss how to choose $\rho$, which depends on $T$, after we derive the regret and sampling-cost violation bounds in Section 6.

This design is motivated by the structure of the learning-error surrogate. For
sufficiently large $t$, the dominant terms depend primarily on the
instantaneous training sample size $n^{\pi}(t)$. Under the buffer constraints, these
terms are minimized when $n^{\pi}(t)$ is stable over time and close to its time
average. This follows from the Cauchy--Schwarz inequality: for a fixed
time-average sample size, temporal fluctuations in $n^{\pi}(t)$ increase the
cumulative contribution of terms involving $1/n^{\pi}(t)$, whereas a constant sample
size yields the smallest value. Accordingly, the time-varying rectangular
constraint allows the admission rate to vary more freely in the early stage so
as to adapt to the time-varying sampling cost, while gradually shrinking the
feasible interval to promote a more stable long-term batch-size profile.

Although~\eqref{eq:dpp-policy} is written in terms of the admission vector
$\boldsymbol{\lambda}(t)$, the objective depends on the control only through the
aggregate admission rate $\Lambda(t)=\sum_{m=1}^M \lambda_m(t)$. Therefore, the
optimization reduces to a one-dimensional per-round problem, which yields
low-complexity online implementation.


\vspace{-0.5em}

\section{Performance Guarantees}
This section analyzes the performance of the proposed ACDPP policy through a
multi-step comparison framework. Our goal is to derive a regret bound for
ACDPP relative to a costless oracle benchmark, together with explicit
constraint-violation guarantees. We first establish a $T$-slot upper bound on
the surrogate penalty induced by ACDPP. We then construct a comparison chain
from the ACDPP-induced learning upper bound to the costless oracle benchmark
and characterize the order of each correction term along this chain. These
ingredients are finally combined to obtain the regret bound, the
sampling-cost violation bound, and an offline method for selecting the
retention horizon $K$ to control the buffer-occupancy violation.

\vspace{-0.5em}
\subsection{Preliminaries}
To analyze the performance of the proposed ACDPP policy, we first introduce the ingredients used in the subsequent regret analysis. We begin with a
costless oracle benchmark, together with the regret and constraint-violation
metrics used to evaluate the proposed policy over a finite horizon. We then
introduce an auxiliary constant-admission policy under the K-step retention
rule, which serves as a tractable intermediate comparison object for both
parameter selection and regret analysis.

\subsubsection{Costless Oracle Benchmark, Regret, and Constraint Violations}

We first introduce an idealized \emph{costless oracle} benchmark together with
the regret and constraint-violation metrics.

Consider the system without the long-term sampling-cost constraint. In this
costless regime, the upper bound in~\eqref{eq:cum-regret-bound} is minimized
when training samples are always fresh and no sample reuse occurs. Under the
per-client buffer constraint $B_m$, this idealized behavior is achieved by the
following \emph{costless oracle policy}: at every round $t$, client $m$ admits
exactly $B_m$ new samples, uses each sample once for local training, and
immediately discards it. This policy satisfies the buffer constraints and
maximizes the number of distinct samples used for learning.

Let $\bar{\mathcal{E}}_T^{\mathrm{ACDPP}}$ denote the upper bound on the cumulative
excess population risk under the proposed ACDPP policy up to round $T$, and let
$\bar{\mathcal{E}}_T^{\mathrm{oracle}}$ denote the corresponding quantity under
the costless oracle policy. We define the regret as
\begin{equation}
\mathrm{Reg}(T)
\triangleq
\bar{\mathcal{E}}_T^{\mathrm{ACDPP}}
-
\bar{\mathcal{E}}_T^{\mathrm{oracle}}.
\label{eq:dpp-regret}
\end{equation}
The regret in~\eqref{eq:dpp-regret} will be used to quantify the
finite-horizon performance gap between the proposed online policy and the
costless oracle benchmark.

To measure the extent to which a policy violates the sampling-cost and buffer
constraints, we define the cumulative sampling-cost violation and cumulative
buffer-occupancy violation up to horizon $T$ as
\begin{equation}
\mathrm{Vio}^{\mathrm S}(T)
\triangleq
\left[
\sum_{t=1}^{T}
\mathbb{E}\!\left[c(t)\sum_{m=1}^M \lambda_m(t)-C\right]
\right]^+,
\label{eq:cum-cost-violation}
\end{equation}
and
\begin{equation}
\mathrm{Vio}^{\mathrm B}(T)
\triangleq
\sum_{m=1}^M
\left[
\sum_{t=1}^{T}\mathbb{E}\!\left[n_m(t)-B_m\right]
\right]^+.
\label{eq:cum-buffer-violation}
\end{equation}
These quantities will be used together with the regret to characterize the
finite-horizon performance of the proposed policy.

\subsubsection{Auxiliary Constant-Admission Policy}

We next introduce an auxiliary comparison policy under the K-step retention
rule.

\begin{definition}[Auxiliary constant-admission policy]
\label{def:constant-admission}
Fix nonnegative admission rates $\{\lambda_m\}_{m=1}^M$. A server-side policy
is called a \emph{constant-admission policy} if, for every round $t$ and client
$m$,
\begin{equation}
\sum_i a_{m,i}(t)=\lambda_m,
\qquad \forall m,t.
\label{eq:constant-admission}
\end{equation}
Combined with the K-step retention policy in
Definition~\ref{def:kstep}, this yields the joint policy
$\pi^{\mathrm{CK}}$.
\end{definition}

Under $\pi^{\mathrm{CK}}$, each client admits a constant number of new samples
per round, while each admitted sample is retained for exactly $K$ rounds.
Consequently, for all sufficiently large $t$, the buffer occupancy at client
$m$ satisfies
\begin{equation}
n_m(t)=K\lambda_m,
\label{eq:ck-buffer-occupancy}
\end{equation}
and the total training sample size is
\begin{equation}
n^{\mathrm{CK}}(t)=K\bar{\Lambda},
\qquad
\bar{\Lambda}\triangleq \sum_{m=1}^M \lambda_m.
\label{eq:ck-training-size}
\end{equation}
Moreover, the cumulative number of distinct samples used in local training up to
round $t$ grows linearly as
\begin{equation}
N^{\mathrm{CK}}(t)=\bar{\Lambda}t.
\label{eq:ck-distinct-samples}
\end{equation}

Substituting~\eqref{eq:ck-buffer-occupancy}--\eqref{eq:ck-distinct-samples}
into Theorem~\ref{thm:regret-neff}, together with the exact finite-horizon
reuse-uniformity factor induced by $\pi^{\mathrm{CK}}$, yields the following
specialization of the cumulative excess-risk bound.

\begin{corollary}
\label{cor:ck-specialization}
For the auxiliary constant-admission policy $\pi^{\mathrm{CK}}$, the cumulative
excess population risk satisfies
\begin{equation}
\mathcal{E}_T(\pi^{\mathrm{CK}})
\le
\bar{\mathcal{E}}_T(\pi^{\mathrm{CK}}),
\label{eq:ck-bound}
\end{equation}
where
\begin{equation}
\begin{aligned}
\bar{\mathcal{E}}_T(\pi^{\mathrm{CK}})
&=
\sum_{t=1}^T
\Bigg[
D\sigma
\sqrt{
\frac{1}{\bar{\Lambda}K}
-
\frac{1}{\bar{\Lambda}t\,\widetilde{\mathcal V}^{\mathrm{CK}}(t)}
}
+
\eta\sigma^2
\left(
\frac{1}{\bar{\Lambda}K}
-
\frac{1}{\bar{\Lambda}t\,\widetilde{\mathcal V}^{\mathrm{CK}}(t)}
\right)
\\
&\quad+
10B
\sqrt{
\frac{\mathrm{Pdim}(\ell\circ\mathcal H)}
{\bar{\Lambda}t\,\widetilde{\mathcal V}^{\mathrm{CK}}(t)}
}
\sqrt{
1+\log\!\left(
\frac{\bar{\Lambda}t}{\mathrm{Pdim}(\ell\circ\mathcal H)}
\right)
}
+
C_2
\Bigg],
\end{aligned}
\label{eq:ck-bound-detail}
\end{equation}
and
\vspace{-0.5em}
\begin{equation}
\widetilde{\mathcal V}^{\mathrm{CK}}(t)
=
\frac{\left(Kt-\frac{K(K-1)}{2}\right)^2}
{t\left(K^2 t-\frac{K(K-1)(2K-1)}{6}\right)}.
\label{eq:vck-exact}
\end{equation}
\end{corollary}

The specialized bound in Corollary~\ref{cor:ck-specialization} depends on the
aggregate admission level $\bar{\Lambda}$, the retention horizon $K$, and the
finite-horizon reuse-uniformity factor
$\widetilde{\mathcal V}^{\mathrm{CK}}(t)$. We therefore choose these parameters
so as to optimize the benchmark bound under the sampling-cost and buffer
constraints.

Under the auxiliary constant-admission policy $\pi^{\mathrm{CK}}$, the average
sampling cost over horizon $T$ must satisfy
\begin{equation}
\bar c_T\,\bar{\Lambda}\le C,
\label{eq:ck-cost-constraint}
\end{equation}
and the per-client buffer constraint requires
\begin{equation}
K\lambda_m\le B_m,
\qquad \forall m.
\label{eq:ck-buffer-constraint}
\end{equation}
Accordingly, the benchmark parameter design is formulated as
\begin{equation}
\label{eq:kstep-opt-problem}
\min_{\{\lambda_m\}_{m=1}^M,\, K\in\mathbb{N}}
\left\{
\bar{\mathcal{E}}_T(\pi^{\mathrm{CK}})
\!:\! 
K\lambda_m \le B_m,\! \forall m,\ 
\bar c_T \!\sum_{m=1}^M \lambda_m \le C,\! 
\lambda_m \ge 0,\! \forall m
\right\}.
\end{equation}

Because $\bar{\mathcal{E}}_T(\pi^{\mathrm{CK}})$ depends on the admission vector
$\{\lambda_m\}_{m=1}^M$ only through the induced aggregate level
$\bar{\Lambda}=\sum_{m=1}^M \lambda_m$, the problem is low-dimensional. In
particular, for any fixed horizon $T$, the benchmark parameters can be obtained
by searching over feasible values of the integer variable $K$ together with the
corresponding admissible aggregate admission levels.

The following corollary gives the resulting closed-form benchmark choice in the
infinite-horizon regime.

\begin{corollary}
\label{cor:optimal-k}
In the infinite-horizon regime, an optimal benchmark solution is obtained by
saturating the per-client buffer constraints, i.e.,
\begin{equation}
\lambda_m^\star=\frac{B_m}{K^\star},
\qquad \forall m,
\end{equation}
where $K^\star$ is chosen as the better of the two nearest feasible integers,
namely
\begin{equation}
\label{eq:optimal-K}
K^\star=
\left\lceil \frac{\bar c\sum_{m=1}^{M}B_m}{C}\right\rceil
\quad \text{or} \quad
K^\star=
\max\!\left\{
\left\lfloor \frac{\bar c\sum_{m=1}^{M}B_m}{C}\right\rfloor,\,1
\right\}.
\end{equation}
\end{corollary}

In the sequel, we adopt the benchmark design in Corollary~\ref{cor:optimal-k}, then the per-client buffer constraint is saturated in steady
state. In particular, for all $t>K$, the buffer occupancy satisfies
\[
n_m(t)=B_m,
\quad \forall m.
\]

%% file: Body/5_1_Performance_guarantee.tex
\vspace{-1em}
\subsection{T-Slot ACDPP Upper Bound}
\label{subsec:t-slot-upper-bound}
We next establish a standard $T$-slot upper bound for the surrogate penalty
induced by the proposed ACDPP policy. This Lyapunov-based result provides the
key control relation used in the subsequent regret analysis.

\begin{lemma}
\label{lem:dpp-T-slot-upper}
The cumulative surrogate penalty under the proposed ACDPP policy satisfies\vspace{-0.5em}
\begin{equation}
\mathbb{E}[L(T+1)-L(1)] + V\sum_{t=1}^{T}\mathbb{E}[p(t)]
\le
BT + V\sum_{t=1}^{T}\bar p^{CK}(t),
\label{eq:dpp-T-slot-upper}
\end{equation}
where\vspace{-0.5em}
\begin{equation}
\begin{aligned}
\bar p^{CK}(t)
&=
D\sigma
\sqrt{
\frac{1}{(K-1)\Lambda_{\min}(t-K)+\bar\Lambda}
-
\frac{1}{\sum_{\tau=1}^{t-1}\Lambda_{\min}(\tau)+\bar\Lambda}
}
\\
&\quad+
\eta\sigma^2
\left(
\frac{1}{(K-1)\Lambda_{\min}(t-K)+\bar\Lambda}
-
\frac{1}{\sum_{\tau=1}^{t-1}\Lambda_{\min}(\tau)+\bar\Lambda}
\right)
\\
&\quad+
10B
\sqrt{
\frac{\mathrm{Pdim}(\ell\circ\mathcal H)}
{\sum_{\tau=1}^{t-1}\Lambda_{\min}(\tau)+\bar\Lambda}
}
\sqrt{
1+\log\!\left(
\frac{\sum_{\tau=1}^{t-1}\Lambda_{\max}(\tau)+\bar\Lambda}
{\mathrm{Pdim}(\ell\circ\mathcal H)}
\right)
}.
\end{aligned}
\label{eq:pC-upper-explicit}
\end{equation}
\end{lemma}
\vspace{-1em}
\begin{proof}
    See Appendix~\ref{app:proof-dpp-T-slot-upper} in~\cite{adaptive_data_admission_retention_sfl_supp}.
\end{proof}
\vspace{-1em}
\subsection{Regret Decomposition}
\label{subsec:regret-decomposition}

To derive the regret bound, we decompose the gap between the ACDPP-induced
learning upper bound and the costless oracle benchmark into a sequence of
finite-horizon correction terms. The key idea is to connect
$\bar{\mathcal E}_T^{\mathrm{ACDPP}}$ to
$\bar{\mathcal E}_T^{\mathrm{oracle}}$ through several intermediate
quantities. In particular, we introduce the auxiliary comparison target
\begin{equation}
\begin{aligned}
\bar p_\infty^{CK}(t)
&\triangleq
D\sigma
\sqrt{
\frac{1}{\bar\Lambda K}
-
\frac{1}{\bar\Lambda t}
}
+
\eta\sigma^2
\left(
\frac{1}{\bar\Lambda K}
-
\frac{1}{\bar\Lambda t}
\right)
\\
&\quad+
10B
\sqrt{
\frac{\mathrm{Pdim}(\ell\circ\mathcal H)}{\bar\Lambda t}
}
\sqrt{
1+\log\!\left(
\frac{\bar\Lambda t}{\mathrm{Pdim}(\ell\circ\mathcal H)}
\right)
}.
\end{aligned}
\label{eq:pc-infty}
\end{equation}
which corresponds to the surrogate penalty of the constant-admission policy in
the infinite-horizon. The overall comparison chain is
\[\bar{\mathcal E}_T^{\mathrm{ACDPP}}
\;\to\;
\sum_{t=1}^{T}\mathbb E[p(t)]
\;\to\;
\sum_{t=1}^{T}\bar p^{CK}(t)
\;\to\;
\sum_{t=1}^{T}\bar p_\infty^{CK}(t)
\;\to\;
\bar{\mathcal E}_T^{\mathrm{oracle}},\]
where each arrow represents a finite-horizon comparison step. Intuitively, the
first step relates the actual ACDPP-induced learning upper bound to its
surrogate penalty; the second step uses the $T$-slot ACDPP upper bound to
compare this surrogate with the finite-horizon comparison form induced by the
auxiliary constant-admission policy; the third step removes the transient
effects caused by the time-varying rectangular action space; and the final step
connects the resulting infinite-horizon comparison target to the costless
oracle benchmark.

More precisely, define\vspace{-0.5em}
\begin{equation}
\Omega_1(T)
\triangleq
\bar{\mathcal E}_T^{\mathrm{ACDPP}}
-
\sum_{t=1}^{T}\mathbb E[p(t)],
\label{eq:omega1-def}
\end{equation}
\begin{equation}
\Omega_2(T)
\triangleq
\sum_{t=1}^{T}\bar p^{CK}(t)
-
\sum_{t=1}^{T}\bar p_\infty^{CK}(t),
\label{eq:omega2-def}
\end{equation}
and\vspace{-0.5em}
\begin{equation}
\Omega_3(T)
\triangleq
\sum_{t=1}^{T}\bar p_\infty^{CK}(t)
-
\bar{\mathcal E}_T^{\mathrm{oracle}}.
\label{eq:omega3-def}
\end{equation}
In addition, the $T$-slot ACDPP upper bound in
Lemma~\ref{lem:dpp-T-slot-upper} yields
\begin{equation}
\sum_{t=1}^{T}\mathbb E[p(t)]
\le
\frac{B}{V}T
+
\sum_{t=1}^{T}\bar p^{CK}(t)
+
\frac{\mathbb E[L(1)]}{V}.
\label{eq:penalty-bridge}
\end{equation}
Combining~\eqref{eq:omega1-def}--\eqref{eq:penalty-bridge}, we obtain
\begin{equation}
\bar{\mathcal E}_T^{\mathrm{ACDPP}}
-
\bar{\mathcal E}_T^{\mathrm{oracle}}
\le
-\Omega_1(T)
+
\Omega_2(T)
+
\Omega_3(T)
+
O\!\left(\frac{T}{V}\right).
\label{eq:omega-master}
\end{equation}
It therefore remains to characterize the orders of the three correction terms.
The technical proofs are deferred to the appendix.

\begin{lemma}
\label{lem:omega1-order}
Under the K-step retention policy,
\begin{equation}
\Omega_1(T)=O(1).
\label{eq:omega1-order}
\end{equation}
\end{lemma}\vspace{-0.5em}
\begin{proof}
    See Appendix~\ref{app:proof-omega1-order} in~\cite{adaptive_data_admission_retention_sfl_supp}.
\end{proof}
\begin{lemma}
\label{lem:omega2-order}
Under the time-varying rectangular constraint~\eqref{eq:lambda-min-max},
\begin{equation}
\Omega_2(T)=O\!\left(\frac{1}{1-\rho}\right).
\label{eq:omega2-order}
\end{equation}
\end{lemma}\vspace{-0.5em}
\begin{proof}
    See Appendix~\ref{app:proof-omega2-order} in~\cite{adaptive_data_admission_retention_sfl_supp}.
\end{proof}
\begin{lemma}
\label{lem:omega3-order}
For any fixed comparison operating point $(\bar\Lambda,K)$ satisfying
$\bar\Lambda K=\bar B$,
\begin{equation}
\Omega_3(T)=O(\log T).
\label{eq:omega3-order}
\end{equation}
\end{lemma}\vspace{-0.5em}
\begin{proof}
    See Appendix~\ref{app:proof-omega3-order} in~\cite{adaptive_data_admission_retention_sfl_supp}.
\end{proof}
\vspace{-1em}
\subsection{Regret and Constraint Guarantees}
\label{subsec:theoretical-guarantees}

Combining the $T$-slot ACDPP upper bound in
Lemma~\ref{lem:dpp-T-slot-upper} with the regret decomposition in
Section~\ref{subsec:regret-decomposition} yields the main guarantees of the
proposed policy: a regret bound relative to the costless oracle benchmark and a
sampling-cost violation bound. We also provide an offline method for selecting
the retention horizon $K$ such that the cumulative buffer-occupancy violation remains
bounded as \(O(1)\).

\begin{theorem}
\label{thm:dpp-regret}
For any fixed ACDPP parameter $V>0$ and any fixed $0<\rho<1$, the regret as defined in~\eqref{eq:dpp-regret} satisfies
\begin{equation}
\mathrm{Reg}(T)
=
O\!\left(\frac{T}{V}\right)
+
O\!\left(\frac{1}{1-\rho}\right)
+
O(\log T).
\label{eq:dpp-regret-rate}
\end{equation}
\end{theorem}
\vspace{-0.5em}
\begin{proof}
Substituting Lemmas~\ref{lem:omega1-order}--\ref{lem:omega3-order}
into~\eqref{eq:omega-master} and absorbing the $O(1)$ term into the constants
proves~\eqref{eq:dpp-regret-rate}.
\end{proof}

\begin{theorem}
\label{thm:dpp-cost-constraint}
For any fixed ACDPP parameter $V>0$ and any fixed $0<\rho<1$, the cumulative sampling-cost constraint
violation satisfies
\begin{equation}
\mathrm{Vio}^{\mathrm S}(T)
=
O\!\left(
\sqrt{VT+\frac{V}{1-\rho}}
\right).
\label{eq:cost-violation-rate}
\end{equation}
\end{theorem}
\begin{proof}
    See Appendix~\ref{app:proof-dpp-cost-constraint} in~\cite{adaptive_data_admission_retention_sfl_supp}.
\end{proof}
\vspace{-1em}
\begin{remark}
Let
$V=T^\gamma,$
and
$\rho = 1-T^{-\phi}$.
Then the cumulative regret and cumulative sampling-cost violation satisfy
\begin{equation}
\mathrm{Reg}(T)
=
O\!\left(T^{1-\gamma}\right)
+
O\!\left(T^\phi\right)
+
O(\log T),
\end{equation}
and
\begin{equation}
\mathrm{Vio}^{\mathrm S}(T)
=
O\!\left(
T^{\frac{1+\gamma}{2}}
+
T^{\frac{\gamma+\phi}{2}}
\right).
\end{equation}

Therefore, for any \(\gamma\in(0,1)\) and \(\phi\in(0,1)\), both
\(\mathrm{Reg}(T)\) and \(\mathrm{Vio}^{\mathrm S}(T)\) are sublinear in \(T\).
\end{remark}
\vspace{-1em}
\begin{remark}
\label{rem:buffer-violation}
The cumulative buffer-occupancy violation is controlled through the offline
selection of the retention horizon $K$. Since $K$ is a time-invariant policy
parameter rather than an online control variable, it must be chosen prior to
execution of the proposed policy. Accordingly, one may evaluate the
buffer-occupancy violation induced under candidate values of $K$ and select a
feasible retention horizon such that the resulting violation remains uniformly
bounded over the horizon of interest. For such a choice of $K$, the cumulative
buffer-occupancy violation satisfies
\begin{equation}
\mathrm{Vio}^{\mathrm B}(T)=O(1).
\end{equation}
Thus, unlike the sampling-cost violation, which is regulated online through the
virtual queue, the buffer-occupancy violation is absorbed into the offline
parameter-selection stage through the choice of $K$.
\end{remark}

\begin{figure*}[t]
    \centering
    \begin{minipage}[t]{0.32\textwidth}
        \centering
        \includegraphics[width=\linewidth,trim=0 0 0 0,clip]{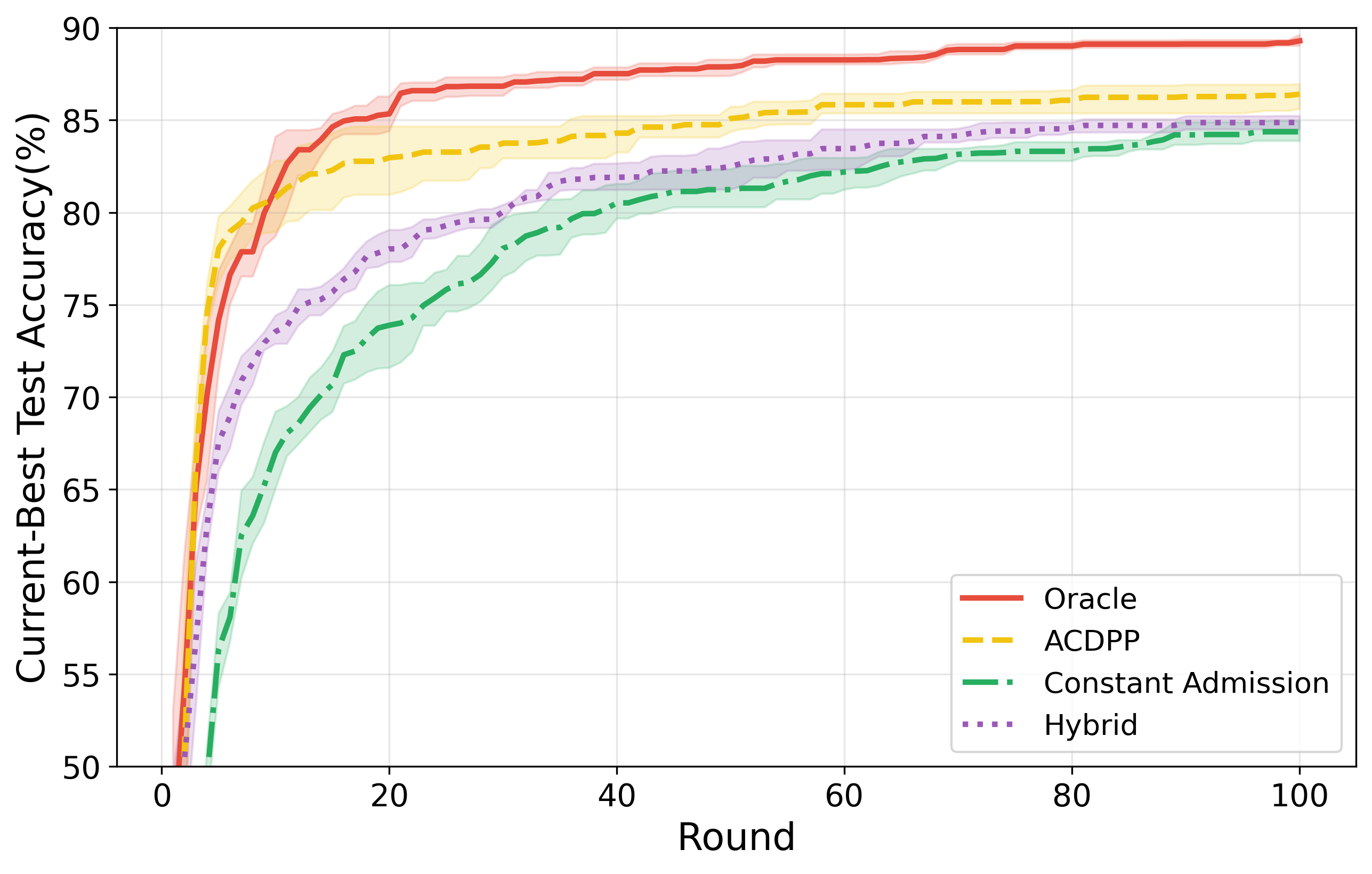}
        
        \small (a) Test accuracy on MNIST (non-i.i.d.)
    \end{minipage}
    \hfill
    \begin{minipage}[t]{0.32\textwidth}
        \centering
        \includegraphics[width=\linewidth,trim=0 0 0 0,clip]{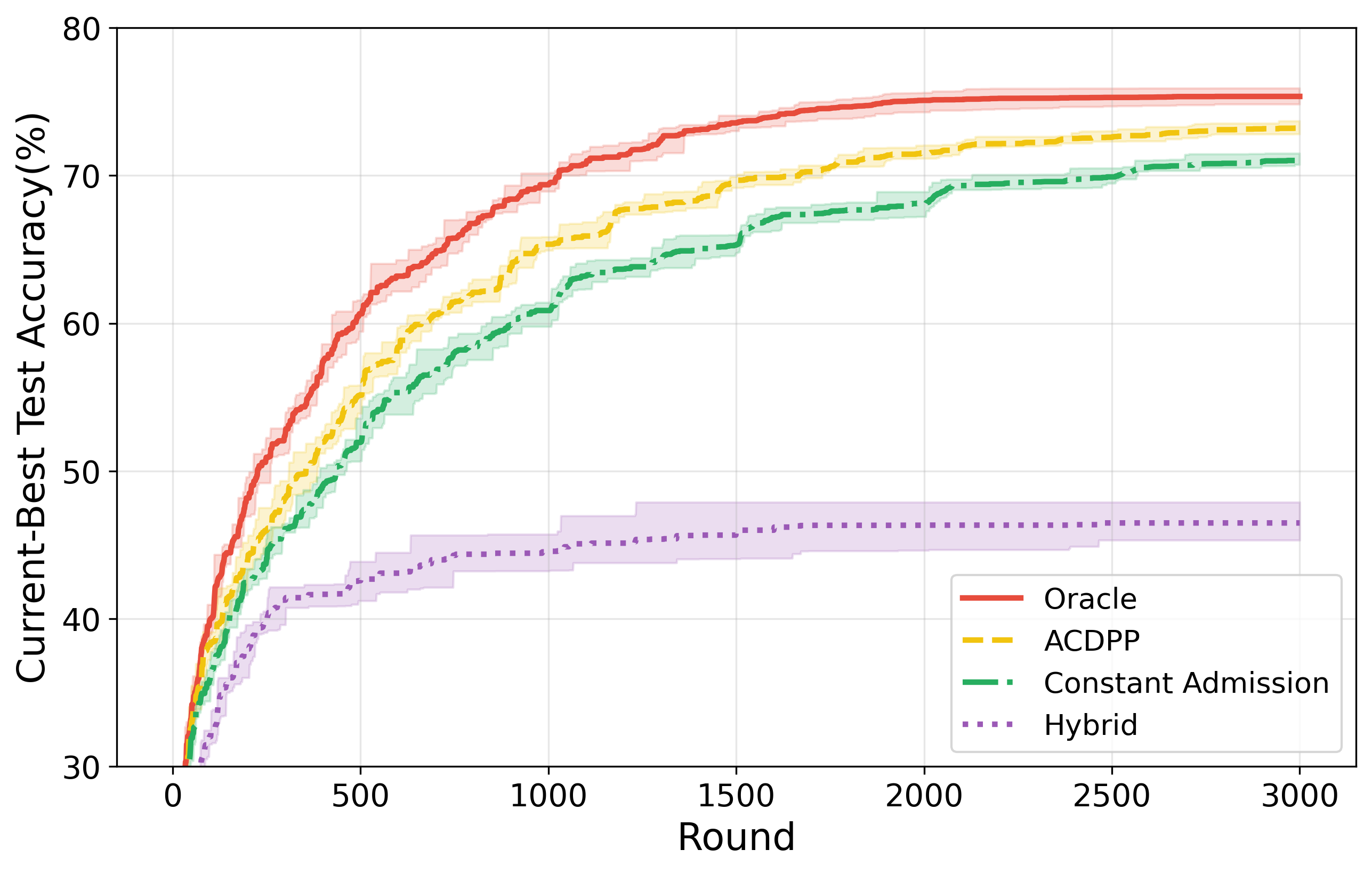}
        
        \small (b) Test 2 on CIFAR-10 (non-i.i.d.)
    \end{minipage}
    \hfill
    \begin{minipage}[t]{0.32\textwidth}
        \centering
        \includegraphics[width=\linewidth,trim=0 0 0 0,clip]{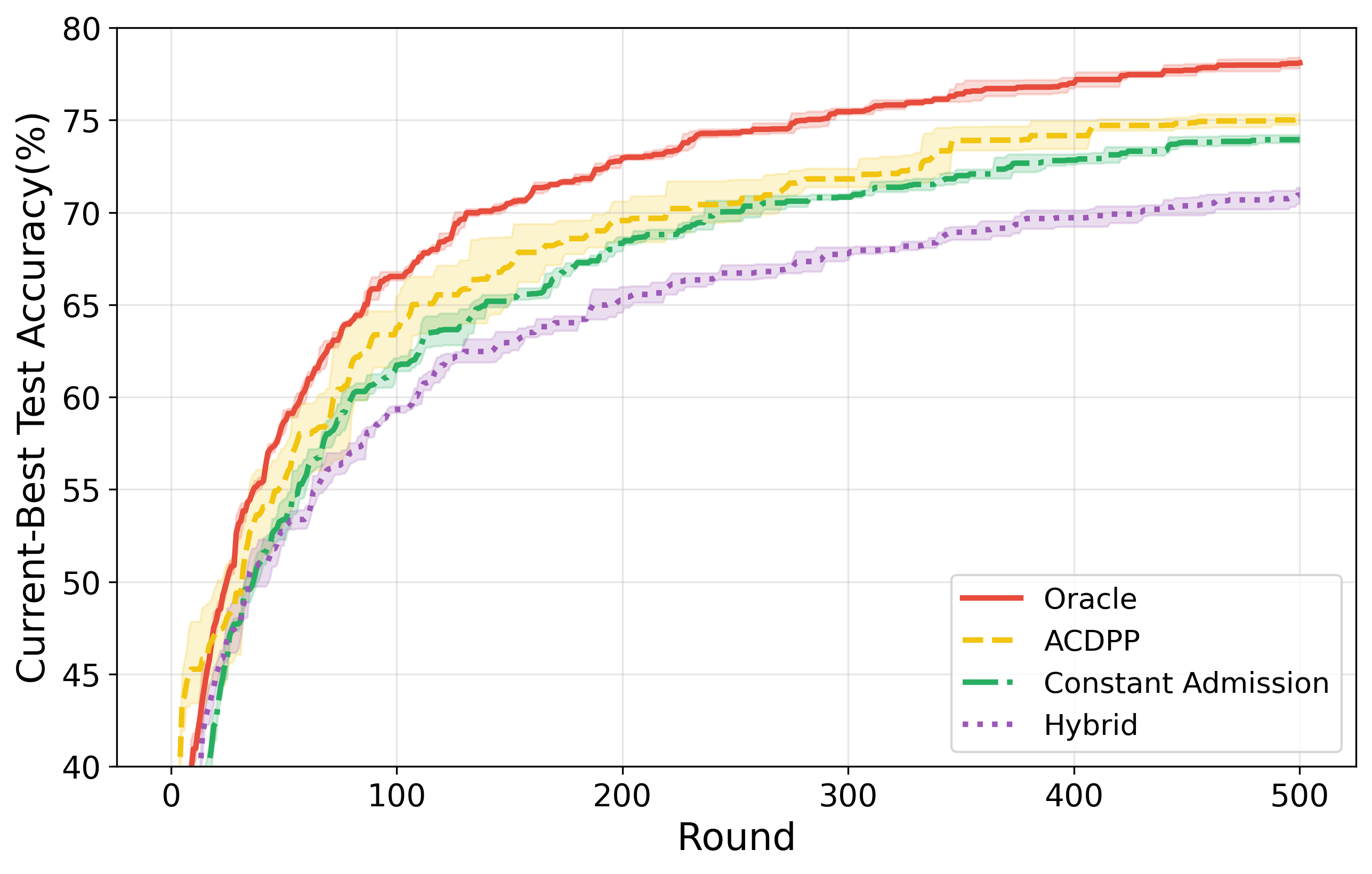}
        
        \small (c) Test accuracy on ImageNette (i.i.d.)
    \end{minipage}
\vspace{-0.5em}
    \caption{Current best test accuracy over communication rounds on MNIST, CIFAR-10, and ImageNette.}
    \label{fig:test_acc_all}
\end{figure*}\vspace{-1em}

\begin{figure*}[!t]
    \centering
    \begin{minipage}[t]{0.48\linewidth}
        \centering
        \includegraphics[width=\linewidth,trim=0 0 0 0,clip]{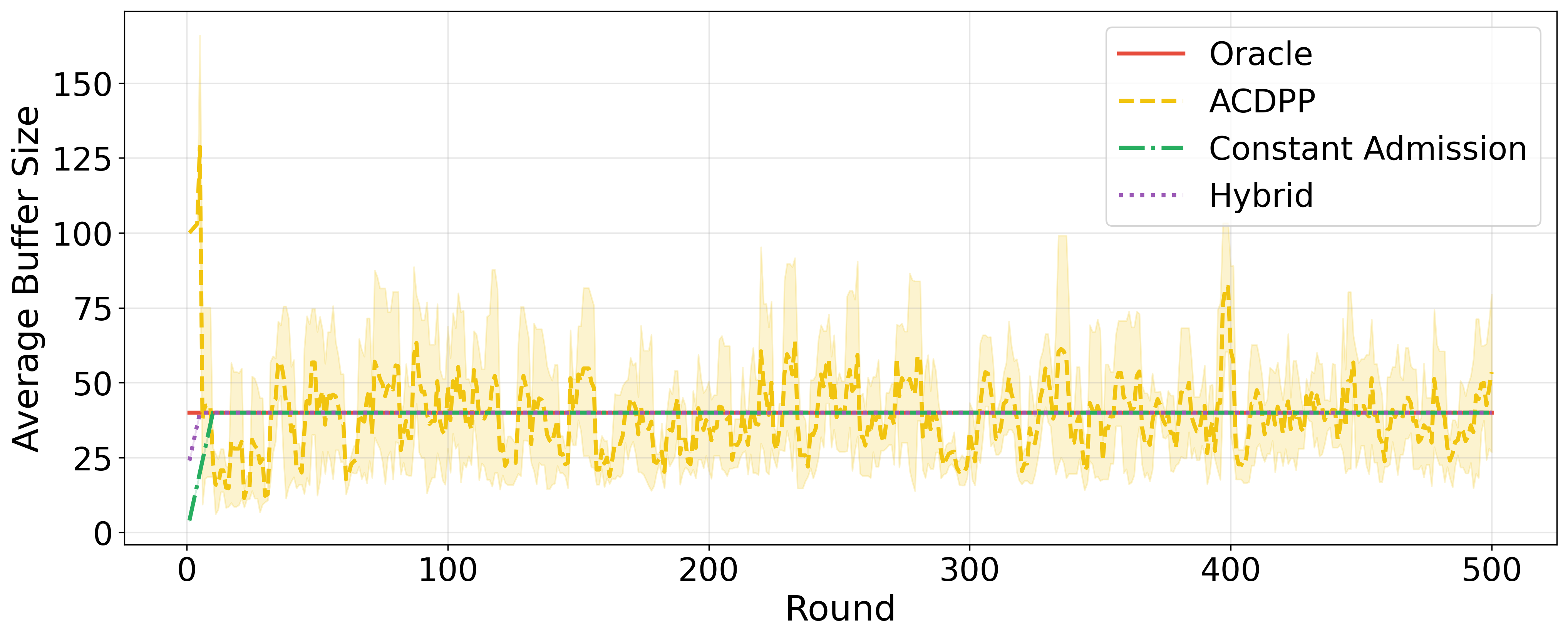}
        \small (a) Buffer occupancy over communication rounds.
    \end{minipage}
    \hfill
    \begin{minipage}[t]{0.48\linewidth}
        \centering
        \includegraphics[width=\linewidth,trim=0 0 0 0,clip]{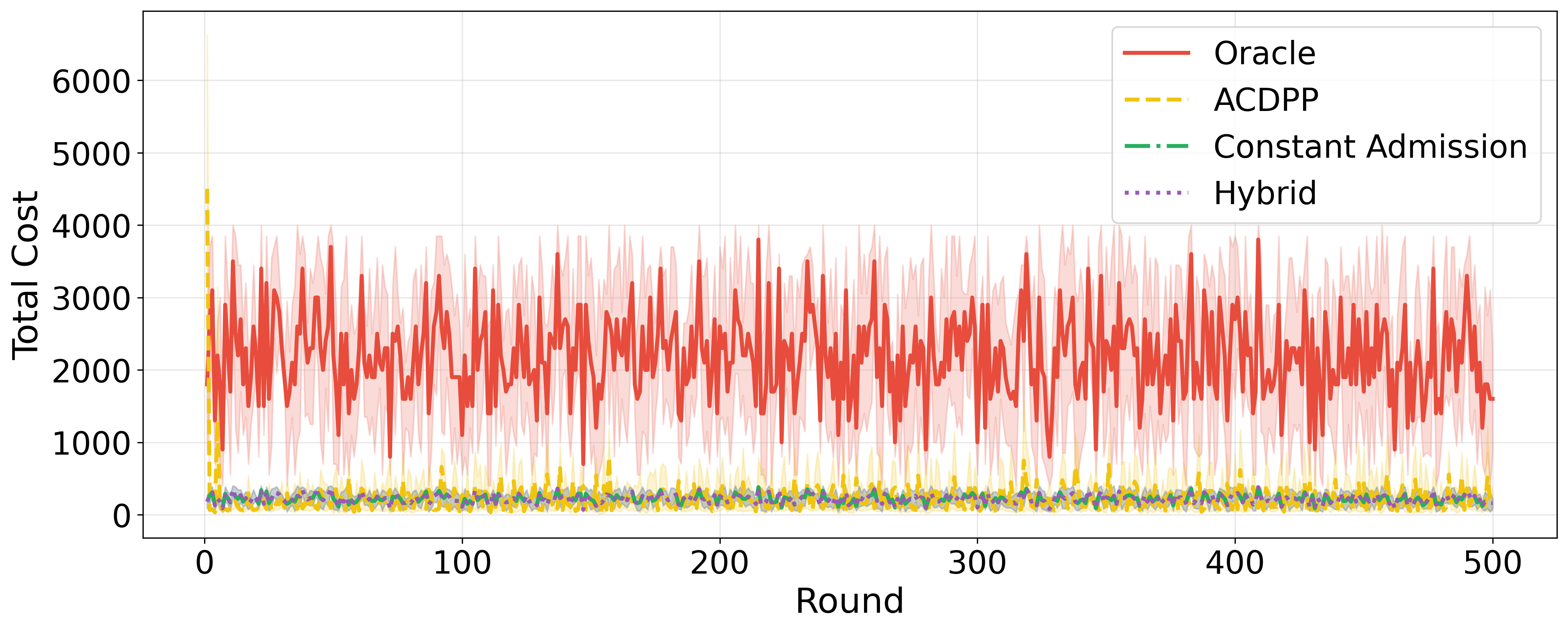}
        
        \small (b) Sampling cost evolution over communication rounds.
    \end{minipage}
\vspace{-0.5em}
    \caption{Buffer occupancy and cost evolution over communication rounds.}
    \label{fig:buffer_cost_horizontal}
\end{figure*}

%% file: Body/7_Experiments.tex
\section{Numerical Evaluation}
\label{sec:experiments}

In addition to the regret analysis in Section 6, we further evaluate the ACDPP policy on three image classification datasets: MNIST~\cite{lecun1998mnist}, CIFAR-10~\cite{krizhevsky2009learning}, and ImageNette~\cite{deng2009imagenet}.\footnote{Code is available at \texttt{https://github.com/zhuoyijoeyzhao/AdaSamplingSFL}.} We compare the proposed ACDPP policy against three baselines: the costless oracle benchmark, the constant-admission policy, and a hybrid policy inspired by~\cite{marfoq2023federated} that combines fresh-sample admission with stale-sample reuse under the same cost budget.
Methods  in~\cite{zhou2020cefl,hu2024energy,gong2023ode,shi2022sofederated,sun2025learn} cannot be included as baselines, since they do not explicitly model the admission-quantity decision as explained in Section 2.2.

\vspace{-1em}
\subsection{Experimental Setup}
\label{subsec:setup}


We consider a federated learning system with $M=10$ clients. For MNIST and
CIFAR-10, we adopt a non-i.i.d. partition in which each client contains samples
from only two out of the ten classes. For ImageNette, the training data are
partitioned i.i.d. and evenly across clients. We adopt LeNet~\cite{lecun2002gradient} as the prediction model for MNIST, while for CIFAR-10 and ImageNette we use ResNet-9~\cite{he2016deep}. For MNIST, the client buffer
capacities are set to
$\{8,9,10,11,12,8,9,10,11,12\}$, and the long-term sampling-cost budget is
$C=55$. For CIFAR-10 and ImageNette, due to the increased task difficulty, both
the client buffer capacities and the sampling-cost budget are set to be,
respectively, $2\times$ and $4\times$ those of MNIST. The per-sample acquisition cost
$c(t)$ is independently drawn from $\mathrm{Unif}[1,10]$. In each communication
round, every client performs $E=10$ local epochs for all three datasets. The
learning rates are set to $\eta=0.5$ for MNIST, $\eta=5\times 10^{-3}$ for
CIFAR-10, and $\eta=5\times 10^{-4}$ for ImageNette. For the hybrid baseline, five clients train only on retained stale samples,
while the other five admit new samples according to the constant-admission rule under
the same cost budget.

\vspace{-1em}
\subsection{Main Results}

We report the current-best test accuracy on all three datasets in
Figure~\ref{fig:test_acc_all}. Each curve is averaged over 10
independent runs, and the shaded region indicates the $80\%$ confidence
interval. For brevity, we present the buffer occupancy and per-round sampling
cost evolution on ImageNette only in Figure~\ref{fig:buffer_cost_horizontal}, since
the same qualitative behavior is observed on the other datasets.

As shown in Figure~\ref{fig:test_acc_all}, the proposed ACDPP policy
consistently outperforms both the hybrid baseline and the
constant-admission baseline, while remaining close
to the oracle policy. On MNIST in Figure~\ref{fig:test_acc_all}(a), ACDPP
improves the average test accuracy by approximately 1.9 and 3.5 percentage
points over the hybrid and constant-admission baselines, respectively. On CIFAR-10 in
Figure~\ref{fig:test_acc_all}(b), the corresponding improvements are 26.7 and
3.5 percentage points. On ImageNette in Figure~\ref{fig:test_acc_all}(c), ACDPP
further improves the average test accuracy by 4.1 and 1.8 percentage points
over the hybrid and constant-admission baselines, respectively. This advantage is also
reflected in the convergence speed visible in Figure~\ref{fig:test_acc_all}. To
quantify this effect, we use the first communication round at which each policy
reaches a target test accuracy, with target levels set to $85\%$ for MNIST,
$70\%$ for CIFAR-10, and $70\%$ for ImageNette. Under this criterion, ACDPP
converges about $2.3\times$ and $2.2\times$ faster than the hybrid baseline on
MNIST and ImageNette, respectively, while on CIFAR-10 the hybrid baseline fails
to reach the target accuracy within the experiment horizon, as seen in
Figure~\ref{fig:test_acc_all}(b). Relative to the constant-admission baseline, the
corresponding speedups are $2.6\times$, $1.7\times$, and $1.3\times$ on MNIST,
CIFAR-10, and ImageNette, respectively. An additional observation is that the hybrid baseline behaves very differently
across tasks. Similar to the interpretation in~\cite{marfoq2023federated}, it
can converge faster than constant-admission on a simple task such as MNIST. However, on more
difficult tasks, especially under distribution heterogeneity, its weakness
becomes much more pronounced. As shown in
Figure~\ref{fig:test_acc_all}(b), on non-i.i.d. CIFAR-10 the hybrid baseline
performs substantially worse than both ACDPP and constant-admission, even though its
fresh-sample component is scheduled according to the optimal constant-admission
constant-admission rule under the same cost budget.

Figure~\ref{fig:buffer_cost_horizontal} further verifies that the proposed ACDPP policy satisfies both the buffer and sampling-cost constraints. In particular, the buffer occupancy in Figure~\ref{fig:buffer_cost_horizontal}(a) remains within the prescribed memory budget, while the stability of the cost-debt queue in Figure~\ref{fig:buffer_cost_horizontal}(b) is consistent with the long-term sampling-cost guarantee established by the Lyapunov analysis.
This reflects the
two-level design of ACDPP: sampling cost is controlled online, while buffer
feasibility is enforced offline through the choice of the retention horizon
$K$. Consequently, ACDPP remains resource-feasible while still achieving
performance close to the costless oracle benchmark.



%% file: Body/6_Conclusion.tex
\vspace{-1em}\section{Conclusion}
\label{sec:conclusion}

In this work, we have studied streaming FL with limited client memory and time-varying sampling costs. We consider a framework of joint server-side sample-admission and client-side memory-management  with the objective of minimizing the cumulative excess population risk under sampling-cost and buffer constraints. To tackle this problem, we first derive a learning-error bound that explicitly captures the roles of instantaneous training sample size, distinct-sample growth, and reuse imbalance through an effective sample size representation. The proposed ACDPP policy leverages a novel learning-error bound that captures the impacts of sample reuse and effective sample size, as well as DPP minimization with a surrogate penalty and time-varying rectangular admission constraints. Through a multi-step comparison analysis together with Lyapunov drift arguments, we have established explicit guarantees in terms of sublinear regret and sampling-cost violation, while showing that the buffer-occupancy violation can be controlled through offline selection of the retention horizon. Experimental results on multiple datasets further demonstrate that the proposed policy achieves strong learning performance in comparison with common baselines and the oracle benchmark. An interesting direction for future work is to extend the present framework to settings with concept drift and to integrate online sample admission with labeling decisions.
\vspace{-1em}

%% file: Body/8_Impact.tex



%% file: Appendix/Appendix_A.tex
\section{Proof of Theorem~\ref{thm:regret-neff}}\label{app: Proof of regret bound}

We prove Theorem~\ref{thm:regret-neff} by combining an optimization recursion
(which yields an optimization-regret bound) with (i) a variance/bias control
term expressed via the instantaneous mini-batch size $n_t$ and the effective
sample size $N_{\mathrm{eff}}^t$, and (ii) a uniform generalization bound
controlled by the pseudo-dimension.

\subsection{Optimization regret bound from recursion}

We first isolate the optimization component of the cumulative
excess-risk bound. The derivation builds on the recursion-based
proof framework in Appendix~B.5 of~\cite{marfoq2023federated}, adapted to
the present full-batch memory-based streaming federated learning
model. In particular, we follow the same empirical-optimization
viewpoint as~\cite{marfoq2023federated}: the optimization recursion
is first established with respect to a policy-induced empirical
objective, and the gap between this empirical objective and the
population objective is then controlled by a separate uniform
generalization bound.

Under policy $\pi$, let
\begin{equation}
\omega_i^\pi(t)
\triangleq
\frac{A_i^\pi(t)}
{\sum_{j=1}^{N^\pi(t)} A_j^\pi(t)},
\qquad i=1,\ldots,N^\pi(t),
\label{eq:omega-def-opt}
\end{equation}
and define the cumulative multiplicity-weighted empirical objective
\begin{equation}
\widehat F_t^{\rm all}(\theta)
\triangleq
\sum_{i=1}^{N^\pi(t)}
\omega_i^\pi(t)\ell(\theta;z_i).
\label{eq:Fhat-all-def}
\end{equation}
This objective is the analogue of the global weighted empirical
objective $L_S^{(\lambda)}$ in~\cite{marfoq2023federated}. It aggregates
all distinct samples used up to round $t$, with weights proportional
to their cumulative reuse counts.

In contrast, the actual local update at round $t$ is performed using
only the samples currently stored in the clients' memories. We denote
the corresponding memory-induced empirical objective by
\begin{equation}
\widehat F_t^{\rm mem}(\theta)
\triangleq
\frac{1}{n^\pi(t)}
\sum_{z\in Q(t)} \ell(\theta;z).
\label{eq:Fhat-mem-def}
\end{equation}

The cumulative empirical optimization error is defined as
\begin{equation}
\mathcal E_T^{\rm opt}(\pi)
\triangleq
\sum_{t=1}^{T}
\mathbb E
\!\left[
\widehat F_t^{\rm all}(\theta^{(t)})
-
\widehat F_t^{{\rm all},\star}
\right].
\label{eq:opt-error-def}
\end{equation}
where
\(
\widehat F_t^{{\rm all},\star}
\triangleq
\min_{\theta\in\Theta}
\widehat F_t^{\rm all}(\theta).
\)
The key round-dependent quantity is therefore the empirical gradient
mismatch between the cumulative multiplicity-weighted objective and
the current memory-induced objective:
\begin{equation}
\bar{\sigma}^2(t)
\triangleq
\mathbb{E}\!\left[
\sup_{\theta\in\Theta}
\left\|
\nabla \widehat F_t^{\rm all}(\theta)
-
\nabla \widehat F_t^{\rm mem}(\theta)
\right\|_2^2
\right].
\label{eq:sigma-bar-def}
\end{equation}
Thus, $\bar{\sigma}^2(t)$ measures how far the gradient used by the
round-$t$ memory-based update deviates from the gradient of the
policy-induced cumulative empirical objective.

The following lemma rewrites the one-step optimization recursion in
Appendix~B.5 of~\cite{marfoq2023federated} under our notation and derives
the corresponding cumulative empirical optimization-error bound.

\begin{lemma}
\label{prop:Ropt-bound-app}
Under Assumptions~\ref{ass:bounded-domain}--\ref{ass:grad-noise},
the iterates $\{\theta^{(t)}\}_{t=1}^T$ generated under policy $\pi$
satisfy the following one-step optimization recursion:
\begin{equation}
\begin{aligned}
\mathbb{E}\!\left[
\|\theta^{(t+1)}\!-\!\theta^\star\|_2^2
\right]
&\!\le\!
\mathbb{E}\!\left[
\|\theta^{(t)}\!-\!\theta^\star\|_2^2
\right]
\!-\!
2\eta\,
\mathbb{E}\!\left[
\widehat F_t^{\rm all}(\theta^{(t)})
\!-\!
\widehat F_t^{{\rm all},\star}
\right]
\\
&
+
2\eta D\,\bar\sigma(t)
+
\eta^2\!\left(
2\bar\sigma^2(t)+C_1
\right)
+
\eta^4C_2,
\end{aligned}
\label{eq:opt-recursion-app}
\end{equation}
where
\begin{equation}
C_1
\triangleq
G\!\left(
5G
+
2LD\sqrt{2(1-E^{-1})}
\right),
\qquad
C_2
\triangleq
4L^2G^2(1-E^{-1}),
\label{eq:C1C2-def}
\end{equation}
and
\begin{equation}
\mathcal E_T^{\rm opt}(\pi)
\le
\frac{D^2}{2\eta}
+
D
\sum_{t=1}^{T}
\bar\sigma(t)
+
\frac{\eta}{2}
\sum_{t=1}^{T}
\!\left(
2\bar\sigma^2(t)
+
C_1
\right)
+
\frac{\eta^3}{2}
C_2T.
\label{eq:Ropt-bound-app}
\end{equation}
\end{lemma}

\begin{proof}
The one-step recursion in
\eqref{eq:opt-recursion-app}
follows directly from the derivation in
Appendix~B.5 of~\cite{marfoq2023federated},
after rewriting the empirical objective and the empirical gradient
mismatch under the notation introduced above.
The higher-order terms are regrouped into the constants
$C_1$ and $C_2$ in
\eqref{eq:C1C2-def}.

Rearranging
\eqref{eq:opt-recursion-app}
gives
\begin{equation}
\begin{aligned}
2\eta
\mathbb E
\!\left[
\widehat F_t^{\rm all}(\theta^{(t)})
-
\widehat F_t^{{\rm all},\star}
\right]
&\le
\mathbb E
\!\left[
\|\theta^{(t)}-\theta^\star\|_2^2
\right]
-
\mathbb E
\!\left[
\|\theta^{(t+1)}-\theta^\star\|_2^2
\right]
\\
&\quad
+
2\eta D\bar\sigma(t)
+
\eta^2
\!\left(
2\bar\sigma^2(t)+C_1
\right)
+
\eta^4C_2.
\end{aligned}
\label{eq:opt-step-app}
\end{equation}

Summing
\eqref{eq:opt-step-app}
over
$t=1,\ldots,T$
telescopes the squared-distance terms:
\[
\sum_{t=1}^{T}\!\!
\left(\!
\mathbb E
\|\theta^{(t)}\!\!-\!\theta^\star\|_2^2
\!-\!
\mathbb E
\|\theta^{(t+1)}\!\!\!-\!\theta^\star\|_2^2\!
\right)
\!=\!
\mathbb E
\|\theta^{(1)}-\theta^\star\|_2^2
\!-\!
\mathbb E
\|\theta^{(T+1)}\!\!-\!\theta^\star\|_2^2
\!\le\!
D^2,
\]
where the last inequality follows from the boundedness of
$\Theta$ and the nonnegativity of squared norms.
Dividing both sides by $2\eta$ yields
\eqref{eq:Ropt-bound-app}.
\end{proof}

\subsection{Variance control via effective sample size}

We next refine the empirical gradient mismatch
$\bar{\sigma}^2(t)$ by explicitly separating the effects of distinct-sample
growth and sample-reuse imbalance. Recall from
Assumption~\ref{ass:grad-noise} that the single-sample population-gradient
variance satisfies
\begin{equation}
\label{eq:pop-grad-var-app}
\mathbb{E}_{z\sim\mathcal{P}}
\big[
\|\nabla_\theta \ell(\theta;z)-\nabla F_{\mathcal P}(\theta)\|_2^2
\big]
\le \sigma^2,
\qquad \forall \theta\in\Theta.
\end{equation}


Under the effective-sample-size variance control induced by repeated sample
reuse, the empirical gradient mismatch at round $t$ satisfies
\begin{equation}
\bar{\sigma}^2(t)
\le
\sigma^2\!\left(
\frac{1}{n^\pi(t)}\!-\!\frac{1}{N^\pi(t)\widetilde{\mathcal V}^\pi(t)}
\right).
\label{eq:sigma-bar-Neff-app}
\end{equation}

Substituting~\eqref{eq:sigma-bar-Neff-app} into
Lemma~\ref{prop:Ropt-bound-app} yields
\begin{equation}
\label{eq:Ropt-Neff-app}
\begin{aligned}
\mathcal{E}_T^{\rm opt}(\pi)
&\le
\frac{D^2}{2\eta}
+
D\sigma\sum_{t=1}^T
\sqrt{
\frac{1}{n^\pi(t)}-\frac{1}{N^\pi(t)\widetilde{\mathcal V}^\pi(t)}
}
\\
&+
\eta\sigma^2\sum_{t=1}^T
\left(
\frac{1}{n^\pi(t)}-\frac{1}{N^\pi(t)\widetilde{\mathcal V}^\pi(t)}
\right)
+\frac{\eta C_1}{2}T+\frac{\eta^3 C_2}{2}T.
\end{aligned}
\end{equation}

\subsection{Generalization term controlled by pseudo-dimension}

We next bound the generalization gap between the population objective
$F_{\mathcal P}$ and the cumulative multiplicity-weighted empirical objective
$\widehat{F}_t^{\rm all}$. It is essentially an
application of Theorem~4.1 in~\cite{marfoq2023federated} to our setting, with
the effective sample size decomposition.
As a result, the generalization bound is expressed explicitly in terms of
cumulative distinct-sample growth and reuse imbalance.
The key point is that, under repeated sample reuse, the effective statistical
complexity is governed not only by the number of distinct samples, but also by
the uniformity of their reuse.

\begin{lemma}[Uniform generalization bound with effective sample size]
\label{lem:gen-Neff-app}
Assume $\ell(\theta;z)\in[0,B]$ for all $\theta\in\Theta$ and $z$. Then for
each round $t$, the expected generalization gap is bounded by
\begin{equation}
\label{eq:gen-gap-Neff-app}
\begin{aligned}
\mathbb{E}\Biggl[
\sup_{\theta\in\Theta}
\bigl|&F_{\mathcal P}(\theta)-\widehat{F}_t^{\rm all}(\theta)\bigr|
\Biggl]
\\&
\le
10B
\sqrt{\frac{\mathrm{Pdim}(\ell\circ\mathcal H)}
{N^\pi(t)\widetilde{\mathcal V}^\pi(t)}}
\sqrt{
1+\log\!\left(
\frac{N^\pi(t)}{\mathrm{Pdim}(\ell\circ\mathcal H)}
\right)
}.
\end{aligned}
\end{equation}
\end{lemma}

\begin{proof}
Recall the policy-induced weights
$\omega_i^\pi(t)$
and the cumulative multiplicity-weighted empirical objective
$\widehat F_t^{\rm all}$ defined in
\eqref{eq:omega-def-opt}--\eqref{eq:Fhat-all-def}.

Consider the loss-composed hypothesis class
\begin{equation}
\mathcal F \triangleq \{z\mapsto \ell(\theta;z):\theta\in\Theta\},
\label{eq:F-class}
\end{equation}
which is uniformly bounded in $[0,B]$ by assumption. By a standard weighted
symmetrization argument, the expected uniform generalization gap is bounded by
twice the weighted empirical Rademacher complexity of $\mathcal F$. Moreover,
the effective normalization is determined by the $\ell_2$ norm of the weight
vector:
\begin{equation}
\sum_{i=1}^{N^\pi(t)}(\omega_i^\pi(t))^2
=
\frac{\sum_{i=1}^{N^\pi(t)}(A_i^\pi(t))^2}
{\left(\sum_{i=1}^{N^\pi(t)}A_i^\pi(t)\right)^2}
=
\frac{1}{N_{\mathrm{eff}}^\pi(t)}
=
\frac{1}{N^\pi(t)\widetilde{\mathcal V}^\pi(t)},
\label{eq:omega-l2}
\end{equation}
Applying the standard pseudo-dimension bound for bounded real-valued function
classes yields
\begin{equation}
\begin{aligned}
\mathbb{E}\!\Biggl[
\sup_{\theta\in\Theta}
\bigl|&F_{\mathcal P}(\theta)-\widehat{F}_t^{\rm all}(\theta)\bigr|
\Biggl]
\\&
\le
10B
\sqrt{\frac{\mathrm{Pdim}(\ell\circ\mathcal H)}
{N^\pi(t)\widetilde{\mathcal V}^\pi(t)}}
\sqrt{
1+\log\!\left(
\frac{N^\pi(t)}{\mathrm{Pdim}(\ell\circ\mathcal H)}
\right)
},
\end{aligned}
\label{eq:gen-final-proof}
\end{equation}
which proves~\eqref{eq:gen-gap-Neff-app}.
\end{proof}

\subsection{Proof of Theorem~\ref{thm:regret-neff}}

\begin{proof}[Proof of Theorem~\ref{thm:regret-neff}]
We combine the optimization-error bound in~\eqref{eq:Ropt-Neff-app} with the
generalization bound in Lemma~\ref{lem:gen-Neff-app}. Specifically,
\eqref{eq:Ropt-Neff-app} controls the cumulative optimization component
$\mathcal{E}_T^{\rm opt}(\pi)$ with respect to the cumulative
multiplicity-weighted empirical objective
$\widehat{F}_t^{\rm all}$ in terms of the policy-induced training
sample size $n^\pi(t)$ and effective sample size
$N^\pi(t)\widetilde{\mathcal V}^\pi(t)$, while
Lemma~\ref{lem:gen-Neff-app} bounds the generalization gap between the
population objective $F_{\mathcal P}$ and
$\widehat{F}_t^{\rm all}$ by a term depending on the pseudo-dimension and the
effective sample size.
Combining the optimization bound in~\eqref{eq:Ropt-Neff-app} with the
generalization bound in Lemma~\ref{lem:gen-Neff-app}
yields the desired cumulative excess-risk upper bound. Finally, absorbing the
initialization term $\frac{D^2}{2\eta}$ and other lower-order constants into
the constant term $C_2$ gives the statement of
Theorem~\ref{thm:regret-neff}.
\end{proof}

%% file: Appendix/Appendix_C.tex
\section{Proof of Lemma~\ref{lem:dpp-T-slot-upper}}
\label{app:proof-dpp-T-slot-upper}

Consider the feasible constant-admission comparison policy associated with the
benchmark operating point $(\bar\Lambda,K)$. For an arbitrary round $t$, let
$\mathbb O(t)$ denote the realized system state induced by the proposed ACDPP
policy over rounds $1,\ldots,t-1$, and suppose that at round $t$ the server
switches to the comparison admission vector $\boldsymbol{\lambda}^{CK}(t)$.

Under the K-step retention policy, each admitted sample remains in memory for
exactly $K$ rounds. Hence the instantaneous batch size and cumulative distinct
sample count under the comparison policy satisfy
\begin{equation}
n^C(t)=n(t)-\Lambda(t-K)+\bar\Lambda,
\qquad
N^C(t)=N(t-1)+\bar\Lambda.
\label{eq:nc-Nc}
\end{equation}
Substituting~\eqref{eq:nc-Nc} into the penalty definition~\eqref{eq:penalty-dpp}
gives the one-slot drift-plus-penalty objective
\begin{equation}
J_t(\boldsymbol{\lambda}^{CK}(t))
=
Z(t)\bigl(c(t)\bar\Lambda-C\bigr)+V\,p^{CK}(t),
\label{eq:J-constant}
\end{equation}
where $p^{CK}(t)$ is the per-round penalty induced by the comparison policy.

Next, since the admissible action space satisfies
\begin{equation}
\Lambda_{\min}(t)\le \Lambda(t)\le \Lambda_{\max}(t),
\label{eq:lambda-interval}
\end{equation}
we have the lower bounds
\begin{equation}
n^C(t)\ge (K-1)\Lambda_{\min}(t-K)+\bar\Lambda,
\label{eq:nc-lb}
\end{equation}
and
\begin{equation}
N^C(t)\ge \sum_{\tau=1}^{t-1}\Lambda_{\min}(\tau)+\bar\Lambda.
\label{eq:Nc-lb}
\end{equation}
Since the penalty function in~\eqref{eq:penalty-dpp} is decreasing in both the
instantaneous batch size and the cumulative distinct-sample count, it follows
that
\begin{equation}
p^{CK}(t)\le \bar p^{CK}(t),
\label{eq:pC-upper}
\end{equation}
where $\bar p^{CK}(t)$ is given by~\eqref{eq:pC-upper-explicit}. Therefore,
\begin{equation}
J_t(\boldsymbol{\lambda}^{CK}(t))
\le
Z(t)\bigl(c(t)\bar\Lambda-C\bigr)+V\bar p^{CK}(t).
\label{eq:J-constant-upper}
\end{equation}

Recall that the one-slot drift-plus-penalty upper bound holds for any feasible
admission vector $\boldsymbol{\lambda}$:
\begin{equation}
\Delta(\mathbb O(t)) + V\mathbb{E}[p(t)\mid \mathbb O(t)]
\le
B + \mathbb{E}[J_t(\boldsymbol{\lambda})\mid \mathbb O(t)].
\label{eq:dpp-upper-anyd}
\end{equation}
By definition, the ACDPP policy chooses
\begin{equation}
\boldsymbol{\lambda}^\star(t) \in \arg\min_{\boldsymbol{\lambda}} J_t(\boldsymbol{\lambda}),
\label{eq:dpp-min-def}
\end{equation}
and hence
\begin{equation}
J_t(\boldsymbol{\lambda}^\star(t))\le J_t(\boldsymbol{\lambda}^{CK}(t)).
\label{eq:dpp-vs-constant-slot}
\end{equation}
Substituting~\eqref{eq:dpp-vs-constant-slot} into~\eqref{eq:dpp-upper-anyd}
yields
\begin{equation}
\Delta(\mathbb O(t)) + V\mathbb{E}[p(t)\mid \mathbb O(t)]
\le
B + \mathbb{E}[J_t(\boldsymbol{\lambda}^{CK}(t))\mid \mathbb O(t)].
\label{eq:dpp-upper-constant-slot}
\end{equation}
Using~\eqref{eq:J-constant-upper}, we obtain
\begin{equation}
\Delta(\mathbb O(t)) \!+\! V\mathbb{E}[p(t)\!\mid \!\mathbb O(t)]
\!\le\! 
B
\!+\!
\mathbb{E}\!\left[Z(t)\bigl(c(t)\bar\Lambda-C\bigr)\!\mid \!\mathbb O(t)\right]
\!+\!
V\bar p^{CK}(t).
\label{eq:dpp-upper-constant-slot-final}
\end{equation}

Now take total expectation. Since the sampling-cost process $\{c(t)\}$ is
independent of the queue state $Z(t)$ under the comparison policy, we have
\begin{equation}
\mathbb{E}\!\left[Z(t)\bigl(c(t)\bar\Lambda-C\bigr)\right]
=
\mathbb{E}[Z(t)]\,\mathbb{E}\!\left[c(t)\bar\Lambda-C\right].
\label{eq:independence-factorization}
\end{equation}
Because the benchmark comparison policy is feasible, it satisfies the long-term
sampling-cost constraint, and therefore
\begin{equation}
\mathbb{E}\!\left[c(t)\bar\Lambda-C\right]\le 0.
\label{eq:comparison-feasible}
\end{equation}
Hence~\eqref{eq:dpp-upper-constant-slot-final} implies
\begin{equation}
\mathbb{E}[L(t+1)-L(t)] + V\mathbb{E}[p(t)]
\le
B + V\bar p^{CK}(t).
\label{eq:dpp-drift-sum}
\end{equation}
Summing~\eqref{eq:dpp-drift-sum} over $t=1,\ldots,T$ gives
\begin{equation}
\mathbb{E}[L(T+1)-L(1)] + V\sum_{t=1}^{T}\mathbb{E}[p(t)]
\le
BT + V\sum_{t=1}^{T}\bar p^{CK}(t),
\end{equation}
which proves~\eqref{eq:dpp-T-slot-upper}.

\section{Proof of Lemma~\ref{lem:omega1-order}}
\label{app:proof-omega1-order}


The only difference between the finite-horizon learning
upper bound $\bar{\mathcal E}_T^{\mathrm{ACDPP}}$ and the cumulative surrogate
penalty $\sum_{t=1}^{T}\mathbb E[p(t)]$ comes from the
reuse-uniformity factor $\widetilde{\mathcal V}^{\pi}(t)$, since the
surrogate penalty replaces
$\frac{1}{N^{\pi}(t)\widetilde{\mathcal V}^{\pi}(t)}$
by
$\frac{1}{N^{\pi}(t)}$.

Here and below, \(O_t(\cdot)\) and $\Theta_t(\cdot)$ denote big-$O$ and big-$\Theta$ with respect to the per-round
index \(t\).
Under the K-step retention policy, only the samples admitted in the most recent
\(K-1\) rounds fail to complete their full \(K\) reuse cycles by round \(t\).
Therefore, the first- and second-moment deficits of the reuse counts are both
\(O_t(1)\), while the cumulative number of distinct samples satisfies
\(N^{\pi}(t)=\Theta_t(t)\). It follows that
\begin{equation}
1-\widetilde{\mathcal V}^{\pi}(t)=O_t(t^{-1}),
\label{eq:app-vtilde-gap-1}
\end{equation}
and hence
\begin{equation}
\frac{1}{N^{\pi}(t)\widetilde{\mathcal V}^{\pi}(t)}
-
\frac{1}{N^{\pi}(t)}
=
\frac{1-\widetilde{\mathcal V}^{\pi}(t)}
{N^{\pi}(t)\widetilde{\mathcal V}^{\pi}(t)}
=
O_t(t^{-2}).
\label{eq:app-vtilde-gap-2}
\end{equation}

Now compare the terms in the learning upper bound with their surrogate
counterparts. Since $n^{\pi}(t)=\Theta_t(1)$ under K-step retention and
$N^{\pi}(t)=\Theta_t(t)$, the quantity
$\frac{1}{n^{\pi}(t)}$$
-$$
\frac{1}{N^{\pi}(t)\widetilde{\mathcal V}^{\pi}(t)}$
remains bounded away from zero for all sufficiently large $t$. Therefore, the
square-root function is locally Lipschitz in a neighborhood of this argument,
and the square-root term differs from its surrogate counterpart by at most
$O_t(t^{-2})$. The same order follows immediately for the linear term.

For the pseudo-dimension term, define
\[
x_t \triangleq \frac{1}{N^{\pi}(t)\widetilde{\mathcal V}^{\pi}(t)},
\qquad
y_t \triangleq \frac{1}{N^{\pi}(t)}.
\]
By~\eqref{eq:app-vtilde-gap-2}, we have $x_t-y_t=O_t(t^{-2})$. Since the map
\[
x\mapsto
10B\sqrt{\mathrm{Pdim}(\ell\circ\mathcal H)\,x}
\sqrt{1+\log\!\left(\frac{1}{x\,\mathrm{Pdim}(\ell\circ\mathcal H)}\right)}
\]
is locally Lipschitz for sufficiently small positive $x$, the pseudo-dimension
term also differs from its surrogate counterpart by at most $O_t(t^{-2})$.

Combining the square-root, linear, and pseudo-dimension terms yields a
per-round discrepancy of order \(O_t(t^{-2})\).
Summing over $t=1,\ldots,T$ gives
\begin{equation}
\bar{\mathcal E}_T^{\mathrm{ACDPP}}
-
\sum_{t=1}^{T}\mathbb E[p(t)]
=
O(1),
\label{eq:app-omega1-order}
\end{equation}
which proves Lemma~\ref{lem:omega1-order}.
\section{Proof of Lemma~\ref{lem:omega2-order}}
\label{app:proof-omega2-order}

Under the time-varying rectangular constraint~\eqref{eq:lambda-min-max}, the
lower bounds on the instantaneous training sample size and the cumulative
number of distinct samples are
\begin{equation}
n_{\mathrm{lb}}(t)
=
K\bar\Lambda-(K-1)\bar\Lambda\,\rho^{\,t-K},
\qquad
N_{\mathrm{lb}}(t)
=
\bar\Lambda t-\bar\Lambda\sum_{\tau=1}^{t-1}\rho^\tau.
\label{eq:app-nNlb}
\end{equation}
It follows from~\eqref{eq:app-nNlb} that, for all \(t\ge K\),
\begin{equation}
\begin{aligned}
\frac{1}{n_{\mathrm{lb}}(t)}
-
\frac{1}{\bar\Lambda K}
&=
\frac{1}{\bar\Lambda}
\left(
\frac{1}{K-(K-1)\rho^{t-K}}
-
\frac{1}{K}
\right)
\\
&=
\frac{(K-1)\rho^{t-K}}
{\bar\Lambda K\bigl(K-(K-1)\rho^{t-K}\bigr)}
\le
\kappa_1\rho^{t-K},
\end{aligned}
\label{eq:app-nlb-recip-detail}
\end{equation}
where the inequality follows from
\(K-(K-1)\rho^{t-K}\ge1\), and
$\kappa_1=\frac{K-1}{\bar\Lambda K}$.

Similarly, because
\[
\sum_{\tau=1}^{t-1}\rho^\tau
=
\frac{\rho(1-\rho^{t-1})}{1-\rho}
\le
\frac{1}{1-\rho},
\]
we have
\[
N_{\mathrm{lb}}(t)=\bar\Lambda t-\delta_t,
\qquad
0\le \delta_t \le \frac{\bar\Lambda}{1-\rho}.
\]
Therefore,
\[
\frac{1}{N_{\mathrm{lb}}(t)}-\frac{1}{\bar\Lambda t}
=
\frac{\delta_t}{(\bar\Lambda t)(\bar\Lambda t-\delta_t)}.
\]
Using the bound on $\delta_t$, we obtain
\begin{equation}
\frac{1}{N_{\mathrm{lb}}(t)}-\frac{1}{\bar\Lambda t}
\le
\frac{\bar\Lambda/(1-\rho)}
{(\bar\Lambda t)\left(\bar\Lambda t-\bar\Lambda/(1-\rho)\right)}
=
\frac{1}{\bar\Lambda\bigl((1-\rho)t^2-t\bigr)}.
\label{eq:app-Nlb-recip}
\end{equation}

We now compare $\bar p^{CK}(t)$ with $\bar p_\infty^{CK}(t)$ term by term. For the
square-root term, define
\[
u_t \triangleq
\frac{1}{n_{\mathrm{lb}}(t)}-\frac{1}{N_{\mathrm{lb}}(t)},
\qquad
v_t \triangleq
\frac{1}{\bar\Lambda K}-\frac{1}{\bar\Lambda t}.
\]
By~\eqref{eq:app-nlb-recip} and~\eqref{eq:app-Nlb-recip},
\[
u_t-v_t
\le
\kappa_1\rho^{\,t-K}
+
\frac{1}{\bar\Lambda\bigl((1-\rho)t^2-t\bigr)}.
\]
Since \(v_t\) is bounded away from zero for all sufficiently large \(t\), the
map \(x\mapsto \sqrt{x}\) is locally Lipschitz around \(v_t\). Therefore, there
exists a constant \(\kappa_2>0\) such that the square-root term in
\(\bar p^{CK}(t)\) differs from its counterpart in \(\bar p_\infty^{CK}(t)\) by
at most
\[
\kappa_2\rho^{\,t-K}
+
\frac{\kappa_2}{(1-\rho)t^2-t}.
\]
The same bound also applies to the linear term.

For the pseudo-dimension term, the only dependence is through the reciprocal
sample-size factor and the logarithmic factor. Using
\eqref{eq:app-Nlb-recip}, together with the local Lipschitz continuity of the
map
\[
x\mapsto
10B\sqrt{\mathrm{Pdim}(\ell\circ\mathcal H)\,x}
\sqrt{1+\log\!\left(\frac{1}{x\,\mathrm{Pdim}(\ell\circ\mathcal H)}\right)},
\]
there exists a constant \(\kappa_3>0\) such that the pseudo-dimension term
differs by at most
\[
\kappa_3\rho^{\,t-K}
+
\frac{\kappa_3}{(1-\rho)t^2-t}.
\]

Combining the three terms, there exists a constant \(\kappa_4>0\) such that
\begin{equation}
\bigl|\bar p^{CK}(t)-\bar p_\infty^{CK}(t)\bigr|
\le
\kappa_4\rho^{\,t-K}
+
\frac{\kappa_4}{(1-\rho)t^2-t},
\qquad t\ge K.
\label{eq:app-pC-pointwise}
\end{equation}

Summing~\eqref{eq:app-pC-pointwise} over \(t=1,\ldots,T\) gives
\begin{equation}
\begin{aligned}
\sum_{t=1}^{T}\bigl|\bar p^{CK}(t)-\bar p_\infty^{CK}(t)\bigr|
&\le
\sum_{t=1}^{K-1}\bigl|\bar p^{CK}(t)-\bar p_\infty^{CK}(t)\bigr|
+
\kappa_4\sum_{t=K}^{T}\rho^{\,t-K}
\\
&\quad+
\kappa_4\sum_{t=1}^{T}\frac{1}{(1-\rho)t^2-t}.
\end{aligned}
\label{eq:app-transient-sum}
\end{equation}
The first term is finite because it contains only finitely many indices. The
second term is bounded by
\[
\kappa_4\sum_{j=0}^{\infty}\rho^j
=
\frac{\kappa_4}{1-\rho}.
\]

For the third term, let
\[
t_0 \triangleq \left\lceil \frac{2}{1-\rho}\right\rceil.
\]
We split the sum as
\[
\sum_{t=1}^{T}\frac{\kappa_4}{(1-\rho)t^2-t}
=
\sum_{t=1}^{t_0-1}\frac{\kappa_4}{(1-\rho)t^2-t}
+
\sum_{t=t_0}^{T}\frac{\kappa_4}{(1-\rho)t^2-t}.
\]
Since \(t_0=O((1-\rho)^{-1})\) and the per-round discrepancy is uniformly
bounded, the first sum contributes \(O((1-\rho)^{-1})\). For the second sum,
when \(t\ge t_0\), we have
\[
t \ge \frac{2}{1-\rho},
\]
which implies
\[
\frac{1}{2}(1-\rho)t^2 \ge t,
\]
or equivalently,
\[
(1-\rho)t^2-t \ge \frac{1}{2}(1-\rho)t^2.
\]
Hence, when $t>= t_0$,
\[
\frac{\kappa_4}{(1-\rho)t^2-t}
\le
\frac{2\kappa_4}{(1-\rho)t^2}.
\]
Therefore,
\[
\sum_{t=t_0}^{T}\frac{\kappa_4}{(1-\rho)t^2-t}
\le
\frac{2\kappa_4}{1-\rho}\sum_{t=t_0}^{\infty}\frac{1}{t^2}
=
O\!\left(\frac{1}{1-\rho}\right).
\]
Combining the two parts yields
\[
\sum_{t=1}^{T}\frac{1}{(1-\rho)t^2-t}
=
O\!\left(\frac{1}{1-\rho}\right).
\]
Hence
\begin{equation}
\sum_{t=1}^{T}\bigl|\bar p^{CK}(t)-\bar p_\infty^{CK}(t)\bigr|
=
O\!\left(\frac{1}{1-\rho}\right),
\label{eq:app-omega2-sum}
\end{equation}
which proves Lemma~\ref{lem:omega2-order}.

\section{Proof of Lemma~\ref{lem:omega3-order}}
\label{app:proof-omega3-order}

Recall that, under the costless oracle policy, the cumulative learning upper
bound takes the form
\begin{equation}
\begin{aligned}
    \bar{\mathcal E}_T^{\mathrm{oracle}}
&=
\sum_{t=1}^{T}
\Bigg[
D\sigma
\sqrt{
\frac{1}{\bar B}
-
\frac{1}{\bar B t}
}
+
\eta\sigma^2
\left(
\frac{1}{\bar B}
-
\frac{1}{\bar B t}
\right)
\\&+
10B
\sqrt{
\frac{\mathrm{Pdim}(\ell\circ\mathcal H)}{\bar B t}
}
\sqrt{
1+\log\!\left(
\frac{\bar B t}{\mathrm{Pdim}(\ell\circ\mathcal H)}
\right)
}
\Bigg],
\label{eq:oracle-upper-bound-explicit}
\end{aligned}
\end{equation}
where $\bar B \triangleq \sum_{m=1}^M B_m$.

Compare the per-round term in $\bar p_\infty^{CK}(t)$ with the corresponding
per-round term in~\eqref{eq:oracle-upper-bound-explicit}. Under the condition
$\bar\Lambda K=\bar B$, the steady-state terms coincide exactly:
\[
\frac{1}{\bar\Lambda K}=\frac{1}{\bar B}.
\]
Thus the only mismatch comes from the distinct-sample-growth terms
\[
\frac{1}{\bar\Lambda t}
\qquad\text{and}\qquad
\frac{1}{\bar B t}.
\]
Their difference is
\begin{equation}
\frac{1}{\bar\Lambda t}-\frac{1}{\bar B t}
=
\frac{1}{\bar\Lambda t}-\frac{1}{\bar\Lambda K\,t}
=
\frac{K-1}{\bar B\,t}
=
O_t(t^{-1}).
\label{eq:app-oracle-gap}
\end{equation}

We again compare the three terms separately. For the square-root term, since
the underlying argument is bounded away from zero for sufficiently large $t$,
the square-root map is locally Lipschitz, and the discrepancy is therefore
$O_t(t^{-1})$. The same order follows immediately for the linear term.

For the pseudo-dimension term, note that
\[
10B
\sqrt{
\frac{\mathrm{Pdim}(\ell\circ\mathcal H)}{\bar\Lambda t}
}
\sqrt{
1+\log\!\left(
\frac{\bar\Lambda t}{\mathrm{Pdim}(\ell\circ\mathcal H)}
\right)
}
\]
and its oracle counterpart differ only through the factor
$1/(\bar\Lambda t)$ versus $1/(\bar B t)$. Since the map
\[
x\mapsto
10B\sqrt{\mathrm{Pdim}(\ell\circ\mathcal H)\,x}
\sqrt{1+\log\!\left(\frac{1}{x\,\mathrm{Pdim}(\ell\circ\mathcal H)}\right)}
\]
is locally Lipschitz for sufficiently small positive $x$, the resulting
discrepancy is of lower order than $t^{-1}$ and can be absorbed into the same
bound.

Hence the per-round difference between $\bar p_\infty^{CK}(t)$ and the oracle term
is $O_t(t^{-1})$. Summing over $t=1,\ldots,T$ yields
\begin{equation}
\sum_{t=1}^{T}\bar p_\infty^{CK}(t)
-
\bar{\mathcal E}_T^{\mathrm{oracle}}
=
O(\log T),
\label{eq:app-omega3-order}
\end{equation}
which proves Lemma~\ref{lem:omega3-order}.
\section{Proof of Theorem~\ref{thm:dpp-cost-constraint}}
\label{app:proof-dpp-cost-constraint}

From Lemma~\ref{lem:dpp-T-slot-upper} and the nonnegativity of the penalty term,
\begin{equation}
\mathbb{E}[L(T+1)-L(1)]
\le
BT + V\sum_{t=1}^{T}\bar p^{CK}(t).
\label{eq:cost-proof-1}
\end{equation}
Using Lemma~\ref{lem:omega2-order}, we have
\[
\sum_{t=1}^{T}\bar p^{CK}(t)
=
\sum_{t=1}^{T}\bar p_\infty^{CK}(t)
+
O\!\left(\frac{1}{1-\rho}\right).
\]
Since each term $\bar p_\infty^{CK}(t)$ is uniformly bounded in $t$, it follows that
\begin{equation}
\sum_{t=1}^{T}\bar p^{CK}(t)
=
O(T)
+
O\!\left(\frac{1}{1-\rho}\right).
\label{eq:cost-proof-2}
\end{equation}
Substituting~\eqref{eq:cost-proof-2} into~\eqref{eq:cost-proof-1} yields
\begin{equation}
\mathbb{E}[L(T+1)]
=
O(VT)
+
O\!\left(\frac{V}{1-\rho}\right).
\label{eq:cost-proof-3}
\end{equation}
Since $L(t)=\frac{1}{2}Z(t)^2$, it follows that
\begin{equation}
\mathbb{E}[Z(T+1)^2]
=
O(VT)
+
O\!\left(\frac{V}{1-\rho}\right),
\label{eq:cost-proof-4}
\end{equation}
and therefore, by Jensen's inequality,
\begin{equation}
\mathbb{E}[Z(T+1)]
=
O\!\left(
\sqrt{VT+\frac{V}{1-\rho}}
\right).
\label{eq:cost-proof-5}
\end{equation}

By the standard virtual-queue argument in Lyapunov optimization
(e.g.,~\cite[Theorem~2.8]{neely2010stability}), the cumulative constraint
violation is upper bounded by $\mathbb{E}[Z(T+1)]$. Therefore,
\begin{equation}
\mathrm{Vio}^{\mathrm S}(T)
\le
O\!\left(
\sqrt{VT+\frac{V}{1-\rho}}
\right),
\label{eq:cost-proof-6}
\end{equation}
which proves~\eqref{eq:cost-violation-rate}.

%% file: references.bib
@book{neely2022stochastic,
  title={Stochastic Network Optimization with Application to Communication and Queueing Systems},
  author={Neely, Michael},
  year={2022},
  publisher={Springer Nature}
}

@article{neely2010stability,
  title={Stability and capacity regions or discrete time queueing networks},
  author={Neely, Michael J},
  journal={arXiv preprint arXiv:1003.3396},
  year={2010}
}

@article{Proof,
  title={Optimizing Age of Information in Networks with
Large and Small Updates},
  author={Zhuoyi Zhao and Vishrant Tripathi and Igor Kadota},
  journal={Technical Report Online: https://tinyurl.com/AoI-Large-Updates-WiOpt},
  year={2025}
}

@article{lecun1998mnist,
    author={LeCun, Yann.},
  title={The MNIST database of handwritten digits},
  year = {1998},
  journal={http://yann. lecun. com/exdb/mnist/}
}

@inproceedings{marfoq2023federated,
  title={Federated learning for data streams},
  author={Marfoq, Othmane and others},
  booktitle={Int. Conf. Artif. Intell. Stat. (AISTATS)},
  year={2023},
}

@inproceedings{mcmahan2017communication,
  title={Communication-efficient learning of deep networks from decentralized data},
  author={McMahan, Brendan and others},
  booktitle={Int. Conf. Artif. Intell. Stat. (AISTATS)},
  year={2017},
}

@inproceedings{karimireddy2020scaffold,
  title={Scaffold: Stochastic controlled averaging for federated learning},
  author={Karimireddy, Sai Praneeth and others},
  booktitle={Int. Conf. Mach. Learn. (ICML)},
  year={2020},
}

@article{tan2022towards,
  title={Towards personalized federated learning},
  author={Tan, Alysa Ziying and others},
  journal={IEEE Trans. Neural Netw. Learn. Syst.},
  volume={34},
  number={12},
  year={2022},
  publisher={IEEE}
}

@ARTICLE{zhou2020cefl,
  author={Zhou, Zhi and others},
  journal={IEEE Internet Things J.}, 
  title={CEFL: Online Admission Control, Data Scheduling, and Accuracy Tuning for Cost-Efficient Federated Learning Across Edge Nodes}, 
  year={2020},
  volume={7},
  number={10},
  pages={9341-9356}
}

@article{hu2024energy,
  title={Energy-efficient federated edge learning with streaming data: A lyapunov optimization approach},
  author={Hu, Chung-Hsuan and others},
  journal={IEEE Trans. Commun.},
  year={2024},
  publisher={IEEE}
}

@article{gong2023ode,
  title={ODE: An online data selection framework for federated learning with limited storage},
  author={Gong, Chen and others},
  journal={IEEE/ACM Trans. Netw.},
  volume={32},
  number={4},
  pages={2794--2809},
  year={2024},
  publisher={IEEE}
}

@article{huynh2025streaming,
  title={Streaming federated learning with Markovian data},
  author={Huynh, Tan-Khiem and others},
  journal={arXiv preprint arXiv:2503.18807},
  year={2025}
}

@inproceedings{zhang2024federating,
  title={Federating from History in Streaming Federated Learning},
  author={Zhang, Ruirui and others},
  booktitle={Int. Symp. Theory Algorithmic Found. Protocol Des. Mobile Netw. Mobile Comput. (MobiHoc)},
  pages={151--160},
  year={2024}
}

@article{shi2022sofederated,
  title={Self-supervised on-device federated learning from unlabeled streams},
  author={Shi, Jiahe and others},
  journal={IEEE Trans. Comput.-Aided Des. Integr. Circuits Syst.},
  volume={42},
  number={12},
  year={2023},
  publisher={IEEE}
}

@inproceedings{sun2025learn,
  title={Learn how to query from unlabeled data streams in federated learning},
  author={Sun, Yuchang and others},
  booktitle={AAAI Conf. Artif. Intell. (AAAI)},
  year={2025}
}

@article{kairouz2021advances,
  title={Advances and open problems in federated learning},
  author={Kairouz, Peter and McMahan, H Brendan},
  journal={Found. Trends Mach. Learn.},
  volume={14},
  number={1-2},
  pages={1--210},
  year={2021},
  publisher={Emerald Publishing Limited}
}

@article{lecun2002gradient,
  title={Gradient-based learning applied to document recognition},
  author={LeCun, Yann and others},
  journal={Proc. IEEE},
  volume={86},
  number={11},
  year={2002},
  publisher={Ieee}
}

@article{bonawitz2019towards,
  title={Towards federated learning at scale: System design},
  author={Bonawitz, Keith and others},
  journal={Proc. Mach. Learn. Syst.},
  volume={1},
  pages={374--388},
  year={2019}
}

@article{wang2019adaptive,
  title={Adaptive federated learning in resource constrained edge computing systems},
  author={Wang, Shiqiang and others},
  journal={IEEE J. Sel. Areas Commun.},
  year={2019},
  publisher={IEEE}
}

@inproceedings{wang2020tackling,
  title={Tackling the objective inconsistency problem in heterogeneous federated optimization},
  author={Wang, Jianyu and others},
  booktitle={Adv. Neural Inf. Process. Syst. (NeurIPS)},
  year={2020}
}

@article{krizhevsky2009learning,
  title={Learning multiple layers of features from tiny images},
  author={Krizhevsky, Alex and others},
  year={2009},
  publisher={Toronto, ON, Canada}
}

@inproceedings{he2016deep,
  title={Deep residual learning for image recognition},
  author={He, Kaiming and others},
  booktitle={IEEE Conf. Comput. Vis. Pattern Recognit. (CVPR)},
  pages={770--778},
  year={2016}
}

@inproceedings{deng2009imagenet,
  title={Imagenet: A large-scale hierarchical image database},
  author={Deng, Jia and others},
  booktitle={IEEE Conf. Comput. Vis. Pattern Recognit. (CVPR)},
  pages={248--255},
  year={2009}
}

@article{gomes2019machine,
  title={Machine learning for streaming data: state of the art, challenges, and opportunities},
  author={Gomes, Heitor Murilo and others},
  journal={ACM SIGKDD Explor. Newsl.},
  volume={21},
  number={2},
  year={2019},
  publisher={ACM New York, NY, USA}
}

@article{cormode2005improved,
  title={An improved data stream summary: the count-min sketch and its applications},
  author={Cormode, Graham and Muthukrishnan, Shan},
  journal={J. Algorithms},
  volume={55},
  number={1},
  year={2005},
  publisher={Elsevier}
}

@article{gama2013evaluating,
  title={On evaluating stream learning algorithms},
  author={Gama, Joao and others},
  journal={Mach. Learn.},
  volume={90},
  number={3},
  pages={317--346},
  year={2013},
  publisher={Springer}
}

@article{wang2024comprehensive,
  title={A comprehensive survey of continual learning: Theory, method and application},
  author={Wang, Liyuan and others},
  journal={IEEE Trans. Pattern Anal. Mach. Intell.},
  volume={46},
  number={8},
  year={2024},
  publisher={IEEE}
}

@article{de2021continual,
  title={A continual learning survey: Defying forgetting in classification tasks},
  author={De Lange, Matthias and others},
  journal={IEEE Trans. Pattern Anal. Mach. Intell.},
  volume={44},
  number={7},
  pages={3366--3385},
  year={2021},
  publisher={IEEE}
}

@misc{adaptive_data_admission_retention_sfl_supp,
  author       = {Zhao, Zhuoyi and Liang, Ben},
  title        = {Adaptive Data Admission and Retention for Streaming Federated Learning},
  year         = {2026},
  howpublished = {\url{https://tinyurl.com/AdaDataSamplingSFL}}
}
